\documentclass[journal]{IEEEtran} %
\usepackage{hyperref}
\hypersetup{colorlinks,allcolors=black}
\usepackage{subfig}
\usepackage[utf8]{inputenc} 
\usepackage{graphicx,url}
\usepackage{booktabs}       
\graphicspath{ {./images/} }
\usepackage{multirow}
\usepackage{hhline}
\usepackage{graphicx}
\usepackage{amssymb,amsmath,mathtools}
\usepackage{amsthm}
\newtheorem{theorem}{Theorem}
\newtheorem{lemma}{Lemma}
\newtheorem{corollary}{Corollary}
\usepackage{bm}                 %

\usepackage{color,soul}
\usepackage[dvipsnames]{xcolor}
\usepackage{bbm}
\usepackage{lipsum,lineno}
\usepackage{float}
\makeatletter
\newif\if@restonecol
\makeatother
\usepackage{amsmath}
\usepackage{kbordermatrix}
\usepackage{mathrsfs}
\usepackage[shortlabels]{enumitem}
\usepackage{algorithm}
\usepackage{algpseudocode}
\usepackage{fancyhdr} 
\usepackage{tabularx}
\usepackage{dsfont}
\usepackage{comment} 
\makeatletter
\renewenvironment{proof}[1][\proofname]{\par
  \pushQED{\qed}%
  \normalfont \topsep6\p@\@plus6\p@\relax
  \trivlist
  \item[\hskip\labelsep
        \itshape
    #1]\ignorespaces
}{%
  \popQED\endtrivlist\@endpefalse
}
\makeatother
\usepackage{tikz,pgfplots}
\usepackage{pgfplots}
\usetikzlibrary{shapes.geometric,backgrounds,patterns, trees}
\usetikzlibrary{3d,decorations.text,shapes.arrows,positioning,fit,backgrounds}
\usetikzlibrary{positioning, decorations.pathmorphing, shapes}
\usetikzlibrary{decorations.pathreplacing}
\usetikzlibrary{shapes.geometric,backgrounds,patterns, trees}
\usetikzlibrary{arrows.meta,
                bending,
                intersections,
                quotes,
                shapes.geometric}
              \usetikzlibrary{automata, positioning}
\usepgfplotslibrary{fillbetween}
\usetikzlibrary{shapes,arrows}
\usetikzlibrary{arrows.meta}
\usetikzlibrary{positioning}
\tikzset{set/.style={draw,circle,inner sep=0pt,align=center}}
\usetikzlibrary{automata, positioning}
  \usetikzlibrary{shapes,shadows}
  \tikzstyle{abstractbox} = [draw=black, fill=white, rectangle,
  inner sep=10pt, style=rounded corners, drop shadow={fill=black,
  opacity=1}]
\tikzstyle{abstracttitle} =[fill=white]
\usetikzlibrary{calc,positioning,shapes.geometric}
\usetikzlibrary{arrows.meta,arrows}
\DeclareMathOperator*{\argmax}{arg\,max}

\usetikzlibrary{matrix}
\tikzstyle{cblue}=[circle, draw, thin,fill=cyan!20, scale=0.8]
\tikzstyle{qgre}=[rectangle, draw, thin,fill=green!20, scale=0.8]
\tikzstyle{rpath}=[ultra thick, red, opacity=0.4]
\tikzstyle{legend_isps}=[rectangle, rounded corners, thin,
                       fill=gray!20, text=blue, draw]

\tikzstyle{legend_overlay}=[rectangle, rounded corners, thin,
                           top color= white,bottom color=green!25,
                           minimum width=2.5cm, minimum height=0.8cm,
                           pinegreen]
\tikzstyle{legend_phytop}=[rectangle, rounded corners, thin,
                          top color= white,bottom color=cyan!25,
                          minimum width=2.5cm, minimum height=0.8cm,
                          royalblue]
\tikzstyle{legend_general}=[rectangle, rounded corners, thin,
                          top color= white,bottom color=lavander!25,
                          minimum width=2.5cm, minimum height=0.8cm,
                          violet]
                          \usetikzlibrary{matrix}

\colorlet{myRed}{red!20}
\tikzset{
  rows/.style 2 args={/utils/temp/.style={row ##1/.append style={nodes={#2}}},
    /utils/temp/.list={#1}},
  columns/.style 2 args={/utils/temp/.style={column ##1/.append style={nodes={#2}}},
    /utils/temp/.list={#1}}}
\usetikzlibrary{backgrounds,calc,shadings,shapes.arrows,shapes.symbols,shadows}
\definecolor{switch}{HTML}{006996}

\makeatletter
\pgfkeys{/pgf/.cd,
  parallelepiped offset x/.initial=2mm,
  parallelepiped offset y/.initial=2mm
}
\pgfdeclareshape{parallelepiped}
{
  \inheritsavedanchors[from=rectangle] 
  \inheritanchorborder[from=rectangle]
  \inheritanchor[from=rectangle]{north}
  \inheritanchor[from=rectangle]{north west}
  \inheritanchor[from=rectangle]{north east}
  \inheritanchor[from=rectangle]{center}
  \inheritanchor[from=rectangle]{west}
  \inheritanchor[from=rectangle]{east}
  \inheritanchor[from=rectangle]{mid}
  \inheritanchor[from=rectangle]{mid west}
  \inheritanchor[from=rectangle]{mid east}
  \inheritanchor[from=rectangle]{base}
  \inheritanchor[from=rectangle]{base west}
  \inheritanchor[from=rectangle]{base east}
  \inheritanchor[from=rectangle]{south}
  \inheritanchor[from=rectangle]{south west}
  \inheritanchor[from=rectangle]{south east}
  \backgroundpath{
    \southwest \pgf@xa=\pgf@x \pgf@ya=\pgf@y
    \northeast \pgf@xb=\pgf@x \pgf@yb=\pgf@y
    \pgfmathsetlength\pgfutil@tempdima{\pgfkeysvalueof{/pgf/parallelepiped
      offset x}}
    \pgfmathsetlength\pgfutil@tempdimb{\pgfkeysvalueof{/pgf/parallelepiped
      offset y}}
    \def\ppd@offset{\pgfpoint{\pgfutil@tempdima}{\pgfutil@tempdimb}}
    \pgfpathmoveto{\pgfqpoint{\pgf@xa}{\pgf@ya}}
    \pgfpathlineto{\pgfqpoint{\pgf@xb}{\pgf@ya}}
    \pgfpathlineto{\pgfqpoint{\pgf@xb}{\pgf@yb}}
    \pgfpathlineto{\pgfqpoint{\pgf@xa}{\pgf@yb}}
    \pgfpathclose
    \pgfpathmoveto{\pgfqpoint{\pgf@xb}{\pgf@ya}}
    \pgfpathlineto{\pgfpointadd{\pgfpoint{\pgf@xb}{\pgf@ya}}{\ppd@offset}}
    \pgfpathlineto{\pgfpointadd{\pgfpoint{\pgf@xb}{\pgf@yb}}{\ppd@offset}}
    \pgfpathlineto{\pgfpointadd{\pgfpoint{\pgf@xa}{\pgf@yb}}{\ppd@offset}}
    \pgfpathlineto{\pgfqpoint{\pgf@xa}{\pgf@yb}}
    \pgfpathmoveto{\pgfqpoint{\pgf@xb}{\pgf@yb}}
    \pgfpathlineto{\pgfpointadd{\pgfpoint{\pgf@xb}{\pgf@yb}}{\ppd@offset}}
  }
}

\makeatletter
\tikzset{anchor/.append code=\let\tikz@auto@anchor\relax,
  add font/.code=%
    \expandafter\def\expandafter\tikz@textfont\expandafter{\tikz@textfont#1},
  left delimiter/.style 2 args={append after command={\tikz@delimiter{south east}
    {south west}{every delimiter,every left delimiter,#2}{south}{north}{#1}{.}{\pgf@y}}}}
\tikzstyle{sms} = [rectangle callout, draw,very thick, rounded corners, minimum height=20pt]
\makeatletter
\tikzset{anchor/.append code=\let\tikz@auto@anchor\relax,
  add font/.code=%
    \expandafter\def\expandafter\tikz@textfont\expandafter{\tikz@textfont#1},
  left delimiter/.style 2 args={append after command={\tikz@delimiter{south east}
    {south west}{every delimiter,every left delimiter,#2}{south}{north}{#1}{.}{\pgf@y}}}}
\tikzstyle{sms} = [rectangle callout, draw,very thick, rounded corners, minimum height=20pt]
\usetikzlibrary{positioning,calc}

\tikzset{l3 switch/.style={
    parallelepiped,fill=switch, draw=white,
    minimum width=0.75cm,
    minimum height=0.75cm,
    parallelepiped offset x=1.75mm,
    parallelepiped offset y=1.25mm,
    path picture={
      \node[fill=white,
        circle,
        minimum size=6pt,
        inner sep=0pt,
        append after command={
          \pgfextra{
            \foreach \angle in {0,45,...,360}
            \draw[-latex,fill=white] (\tikzlastnode.\angle)--++(\angle:2.25mm);
          }
        }
      ]
       at ([xshift=-0.75mm,yshift=-0.5mm]path picture bounding box.center){};
    }
  },
  ports/.style={
    line width=0.3pt,
    top color=gray!20,
    bottom color=gray!80
  },
  rack switch/.style={
    parallelepiped,fill=white, draw,
    minimum width=1.25cm,
    minimum height=0.25cm,
    parallelepiped offset x=2mm,
    parallelepiped offset y=1.25mm,
    xscale=-1,
    path picture={
      \draw[top color=gray!5,bottom color=gray!40]
      (path picture bounding box.south west) rectangle
      (path picture bounding box.north east);
      \coordinate (A-west) at ([xshift=-0.2cm]path picture bounding box.west);
      \coordinate (A-center) at ($(path picture bounding box.center)!0!(path
        picture bounding box.south)$);
      \foreach \x in {0.275,0.525,0.775}{
        \draw[ports]([yshift=-0.05cm]$(A-west)!\x!(A-center)$)
          rectangle +(0.1,0.05);
        \draw[ports]([yshift=-0.125cm]$(A-west)!\x!(A-center)$)
          rectangle +(0.1,0.05);
       }
      \coordinate (A-east) at (path picture bounding box.east);
      \foreach \x in {0.085,0.21,0.335,0.455,0.635,0.755,0.875,1}{
        \draw[ports]([yshift=-0.1125cm]$(A-east)!\x!(A-center)$)
          rectangle +(0.05,0.1);
      }
    }
  },
  server/.style={
    parallelepiped,
    fill=white, draw,
    minimum width=0.35cm,
    minimum height=0.75cm,
    parallelepiped offset x=3mm,
    parallelepiped offset y=2mm,
    xscale=-1,
    path picture={
      \draw[top color=gray!5,bottom color=gray!40]
      (path picture bounding box.south west) rectangle
      (path picture bounding box.north east);
      \coordinate (A-center) at ($(path picture bounding box.center)!0!(path
        picture bounding box.south)$);
      \coordinate (A-west) at ([xshift=-0.575cm]path picture bounding box.west);
      \draw[ports]([yshift=0.1cm]$(A-west)!0!(A-center)$)
        rectangle +(0.2,0.065);
      \draw[ports]([yshift=0.01cm]$(A-west)!0.085!(A-center)$)
        rectangle +(0.15,0.05);
      \fill[black]([yshift=-0.35cm]$(A-west)!-0.1!(A-center)$)
        rectangle +(0.235,0.0175);
      \fill[black]([yshift=-0.385cm]$(A-west)!-0.1!(A-center)$)
        rectangle +(0.235,0.0175);
      \fill[black]([yshift=-0.42cm]$(A-west)!-0.1!(A-center)$)
        rectangle +(0.235,0.0175);
    }
  },
}

\usetikzlibrary{calc, shadings, shadows, shapes.arrows}

\tikzset{%
  interface/.style={draw, rectangle, rounded corners, font=\LARGE\sffamily},
  ethernet/.style={interface, fill=yellow!50},
  serial/.style={interface, fill=green!70},
  speed/.style={sloped, anchor=south, font=\large\sffamily},
  route/.style={draw, shape=single arrow, single arrow head extend=4mm,
    minimum height=1.7cm, minimum width=3mm, white, fill=switch!20,
    drop shadow={opacity=.8, fill=switch}, font=\tiny}
}
\newcommand*{\shift}{1.3cm}
\newcommand{\Crossk}{$\mathbin{\tikz [x=1.2ex,y=1.2ex,line width=.1ex, black] \draw (0,0) -- (1,1) (0,1) -- (1,0);}$}%
\newcommand*{\router}[1]{
\begin{tikzpicture}
  \coordinate (ll) at (-3,0.5);
  \coordinate (lr) at (3,0.5);
  \coordinate (ul) at (-3,2);
  \coordinate (ur) at (3,2);
  \shade [shading angle=90, left color=switch, right color=white] (ll)
    arc (-180:-60:3cm and .75cm) -- +(0,1.5) arc (-60:-180:3cm and .75cm)
    -- cycle;
  \shade [shading angle=270, right color=switch, left color=white!50] (lr)
    arc (0:-60:3cm and .75cm) -- +(0,1.5) arc (-60:0:3cm and .75cm) -- cycle;
  \draw [thick] (ll) arc (-180:0:3cm and .75cm)
    -- (ur) arc (0:-180:3cm and .75cm) -- cycle;
  \draw [thick, shade, upper left=switch, lower left=switch,
    upper right=switch, lower right=white] (ul)
    arc (-180:180:3cm and .75cm);
  \node at (0,0.5){\color{blue!60!black}\Huge #1};
  \begin{scope}[yshift=2cm, yscale=0.28, transform shape]
    \node[route, rotate=45, xshift=\shift] {\strut};
    \node[route, rotate=-45, xshift=-\shift] {\strut};
    \node[route, rotate=-135, xshift=\shift] {\strut};
    \node[route, rotate=135, xshift=-\shift] {\strut};
  \end{scope}
\end{tikzpicture}}

\makeatletter
\pgfdeclareradialshading[tikz@ball]{cloud}{\pgfpoint{-0.275cm}{0.4cm}}{%
  color(0cm)=(tikz@ball!75!white);
  color(0.1cm)=(tikz@ball!85!white);
  color(0.2cm)=(tikz@ball!95!white);
  color(0.7cm)=(tikz@ball!89!black);
  color(1cm)=(tikz@ball!75!black)
}
\tikzoption{cloud color}{\pgfutil@colorlet{tikz@ball}{#1}%
  \def\tikz@shading{cloud}\tikz@addmode{\tikz@mode@shadetrue}}
\makeatother

\tikzset{my cloud/.style={
     cloud, draw, aspect=2,
     cloud color={gray!5!white}
  }
}

\renewcommand\thetable{\arabic{table}}
\usepackage{etoolbox}
\makeatletter
\patchcmd{\@makecaption}
  {\scshape}
  {}
  {}
  {}
\makeatother

\begin{document}
\bstctlcite{MyBSTcontrol}
\title{Intrusion Prevention through Optimal Stopping}

\author{\IEEEauthorblockN{Kim Hammar \IEEEauthorrefmark{2}\IEEEauthorrefmark{3} and Rolf Stadler\IEEEauthorrefmark{2}\IEEEauthorrefmark{3}}

 \IEEEauthorblockA{\IEEEauthorrefmark{2}
Division of Network and Systems Engineering, KTH Royal Institute of Technology, Sweden
 }\\
 \IEEEauthorblockA{\IEEEauthorrefmark{3} KTH Center for Cyber Defense and Information Security, Sweden \\
Email: \{kimham, stadler\}@kth.se%
\\
\today
}
}

\markboth{\copyright 2022 IEEE; This work has been submitted to the IEEE for possible publication. Copyright may be transferred without notice.}%
{}
\maketitle
\begin{abstract}
We study automated intrusion prevention using reinforcement learning. Following a novel approach, we formulate the problem of intrusion prevention as an (optimal) multiple stopping problem. This formulation gives us insight into the structure of optimal policies, which we show to have threshold properties. For most practical cases, it is not feasible to obtain an optimal defender policy using dynamic programming. We therefore develop a reinforcement learning approach to approximate an optimal threshold policy.
We introduce \textsc{T-SPSA}, an efficient reinforcement learning algorithm that learns threshold policies through stochastic approximation. We show that \textsc{T-SPSA} outperforms state-of-the-art algorithms for our use case. Our overall method for learning and validating policies includes two systems: a simulation system where defender policies are incrementally learned and an emulation system where statistics are produced that drive simulation runs and where learned policies are evaluated. We show that this approach can produce effective defender policies for a practical IT infrastructure.
\end{abstract}

\begin{IEEEkeywords}
Network security, automation, optimal stopping, reinforcement learning, Markov Decision Process, MDP, POMDP.
\end{IEEEkeywords}

\IEEEpeerreviewmaketitle

\section{Introduction}
An organization's security strategy has traditionally been defined, implemented, and updated by domain experts \cite{int_prevention}. Although this approach can provide basic security for an organization's communication and computing infrastructure, a growing concern is that infrastructure update cycles become shorter and attacks increase in sophistication \cite{cyber_threat_landscape,intelligent_attacks_foi}. Consequently, the security requirements become increasingly difficult to meet. To address this challenge, significant efforts have started to automate security frameworks and the process of obtaining effective security policies. Examples of this research include: automated creation of threat models \cite{mal_pontus}; computation of defender policies using dynamic programming and control theory \cite{dp_security_1,Miehling_control_theoretic_approaches_summary}; computation of exploits and corresponding defenses through evolutionary methods \cite{armsrace_malware}; identification of infrastructure vulnerabilities through attack simulations and threat intelligence \cite{wagner_automated_segmentation, threat_intel_misp}; computation of defender policies through game-theoretic methods \cite{nework_security_alpcan, serkan_gyorgy_game}; and use of machine learning techniques to estimate model parameters and policies \cite{hammar_stadler_cnsm_20, hammar_stadler_cnsm_21}.

In this paper, we present a novel approach to automatically learn defender policies. We apply this approach to an \textit{intrusion prevention} use case. Here, we use the term "intrusion prevention'' as suggested in the literature, e.g. in \cite{int_prevention}. It means that a defender prevents an attacker from reaching its goal, rather than preventing it from accessing any part of the infrastructure.

Our use case involves the IT infrastructure of an organization (see Fig. \ref{fig:system2}). The operator of this infrastructure, which we call the defender, takes measures to protect it against a possible attacker while, at the same time, providing a service to a client population. The infrastructure includes a public gateway through which the clients access the service and which also is open to a possible attacker. The attacker decides when to start an intrusion and then executes a sequence of actions that includes reconnaissance and exploits. Conversely, the defender aims at preventing intrusions and maintaining service to its clients. It monitors the infrastructure and can defend it by taking defensive actions, which can prevent a possible attacker but also incur costs. What makes the task of the defender difficult is the fact that it lacks direct knowledge of the attacker's actions and must infer that an intrusion occurs from monitoring data.
\begin{figure}
  \centering
  \scalebox{0.93}{
    \input{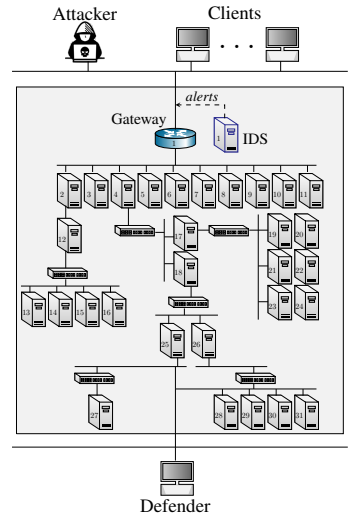}
    }
    \caption{The IT infrastructure and the actors in the use case.}
    \label{fig:system2}
\end{figure}

We study the use case within the framework of discrete-time dynamical systems. Specifically, we formulate the problem of finding an optimal defender policy as an \textit{(optimal) multiple stopping problem}. In this formulation, the defender can take a finite number of \textit{stops}. Each stop is associated with a defensive action and the objective is to decide the optimal times when to stop. This approach gives us insight into the structure of optimal defender policies through the theory of dynamic programming and optimal stopping \cite{bellman1957markovian,wald}. In particular, we show that an optimal \textit{multi-threshold policy} exists that can be efficiently computed and implemented.

The structure of optimal policies in dynamical systems is a well studied area \cite{puterman,krishnamurthy_2016}. However, it has not been considered in prior research on automated intrusion prevention \cite{hammar_stadler_cnsm_20,elderman, schwartz_2020, oslo_pentest_rl, microsoft_red_teaming, ridley_ml_defense, muzero_sdn, deep_hierarchical_rl_pentest, pentest_rl_rohit, adaptive_cyber_defense_pomdp_rl,atmos,sdn_rl_ddos,deep_air}. Further, although the optimal stopping problem frequently is used to formulate problems in the fields of finance and communication systems \cite{optimal_multiple_stopping_finance_1,optimal_multiple_stopping_finance_2,kleinberg_multi_secretary,optimal_multiple_stopping_social_media_1,Nakai1985,optimal_stopping_finance,roy_threshold,tartakovsky_1,kurt_rl,hammar_stadler_cnsm_21}, to the best of our knowledge, formulating intrusion prevention as a multiple stopping problem is a novel approach.

Since the defender can access only a set of infrastructure metrics and does not directly observe the attacker, we use a Partially Observed Markov Decision Process (POMDP) to model the multiple stopping problem. An optimal policy for a POMDP can be obtained through two main methods: dynamic programming and reinforcement learning. In our case, dynamic programming is not feasible due to the size of the POMDP \cite{pspace_complexity}. Therefore, we use a reinforcement learning approach to obtain the defender policy. We simulate a long series of POMDP episodes whereby the defender continuously updates its policy based on outcomes of previous episodes. To update the policy, we introduce \textsc{T-SPSA}, a reinforcement learning algorithm that exploits the threshold structure of optimal policies. We show that \textsc{T-SPSA} efficiently learns a near-optimal policy despite the high complexity of computing optimal policies for general POMDPs \cite{pspace_complexity}.

Our method for learning and validating policies includes building two systems (see Fig. \ref{fig:method}). First, we develop an \textit{emulation system} where key functional components of the target infrastructure are replicated. In this system, we run attack scenarios and defender responses. These runs produce system metrics and logs that we use to estimate empirical distributions of infrastructure metrics, which are needed to simulate POMDP episodes. Second, we develop a \textit{simulation system} where POMDP episodes are executed and policies are incrementally learned. Finally, the policies are extracted and evaluated in the emulation system and possibly implemented in the target infrastructure (see Fig. \ref{fig:method}). In short, the emulation system is used to provide the statistics needed to simulate the POMDP and to evaluate policies, whereas the simulation system is used to learn policies.

We make three contributions with this paper. First, we formulate intrusion prevention as a problem of multiple stopping. This novel formulation allows us a) to derive properties of an optimal defender policy using results from dynamic programming and optimal stopping; and b) to approximate an optimal policy for a non-trivial infrastructure configuration. Second, we present a reinforcement learning approach to obtain policies in an emulated infrastructure. With this approach, we narrow the gap between the evaluation environment and a scenario playing out in a real system. We also address a limitation of many related works, which rely on simulations solely to evaluate policies \cite{hammar_stadler_cnsm_20,armsrace_malware,elderman,oslo_pentest_rl,microsoft_red_teaming,rl_cyberdefense_heartbleed,schwartz_2020}. Third, we present \textsc{T-SPSA}, an efficient reinforcement learning algorithm that exploits the threshold structure of optimal policies and outperforms state-of-the-art algorithms for our use case.

We conclude this section with remarks about the context of this research and the practical relevance of the results in this paper. The objective of our line of research is to construct a mathematical and conceptual framework, validated by an experimental environment, that produces defender policies for realistic scenarios through self-learning. We are engaged in a program with high potential reward that will need many years of investigation. This paper provides an important result and milestone in this program.

From a practical point of view, the main question the paper answers is this: at which points in time should a defender take defensive actions given periodic but limited observational data? The paper proposes a fundamental framework to study this question. We show theoretically and experimentally that the optimal action times can be obtained through thresholds that the framework predicts and which can be efficiently implemented in a real system.
\section{The Intrusion Prevention Use Case}\label{sec:use_case}
\begin{figure}
  \centering
  \scalebox{0.95}{
    \input{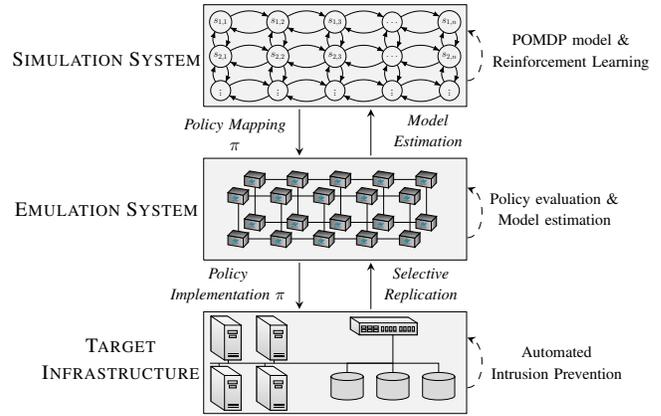}
 }
    \caption{Our approach for finding and evaluating intrusion prevention policies.}
    \label{fig:method}
\end{figure}
We consider an intrusion prevention use case that involves the IT infrastructure of an organization. The operator of this infrastructure, which we call the defender, takes measures to protect it against an attacker while, at the same time, providing a service to a client population (Fig. \ref{fig:system2}). The infrastructure includes a set of servers that run the service and an intrusion detection system (IDS) that logs events in real-time. Clients access the service through a public gateway, which also is open to the attacker.

We assume that the attacker intrudes into the infrastructure through the gateway, performs reconnaissance, and exploits found vulnerabilities, while the defender continuously monitors the infrastructure through accessing and analyzing IDS statistics and login attempts at the servers. The defender can take a fixed number of defensive actions to prevent the attacker. A defensive action is for example to revoke user certificates in the infrastructure, which will recover user accounts compromised by the attacker. It is assumed that the defender takes the defensive actions in a predetermined order. The final action that the defender can take is to block all external access to the gateway. As a consequence of this action, the service as well as any ongoing intrusion are disrupted.

In deciding when to take defensive actions, the defender has two objectives: (\textit{i}) maintain service to its clients; and (\textit{ii}), keep a possible attacker out of the infrastructure. The optimal policy for the defender is to monitor the infrastructure and maintain service until the moment when the attacker enters through the gateway, at which time the attacker must be prevented by taking defensive actions. The challenge for the defender is to identify the precise time when this moment occurs.

In this work, we model the attacker as an agent that starts the intrusion at a random point in time and then takes a predefined sequence of actions, which includes reconnaissance to explore the infrastructure and exploits to compromise servers.

We study the use case from the defender’s perspective. The evolution of the system state and the actions by the defender are modeled with a discrete-time Partially Observed Markov Decision Process (POMDP). The reward function of this process encodes the benefit of maintaining service and the loss of being intruded. Finding an optimal defender policy thus means maximizing the expected reward.
\section{Theoretical Background}\label{sec:theory_background}
This section covers the preliminaries on Markov decision processes, reinforcement learning, and optimal stopping.
\subsection{Markov Decision Processes}\label{sec:mdps}
A Markov Decision Process (MDP) models the control of a discrete-time dynamical system and is defined by a seven-tuple $\mathcal{M} = \langle \mathcal{S}, \mathcal{A}, \mathcal{P}^{a_t}_{s_t,s_{t+1}}, \mathcal{R}^{a_t}_{s_t,s_{t+1}}, \gamma, \rho_1, T \rangle$ \cite{bellman1957markovian,puterman}. $\mathcal{S}$ denotes the set of states and $\mathcal{A}$ denotes the set of actions. $\mathcal{P}^{a_t}_{s_t,s_{t+1}}$ refers to the probability of transitioning from state $s_t$ to state $s_{t+1}$ when taking action $a_t$ (Eq. \ref{eq:mdp_prob_1}), which has the Markov property $\mathbb{P}\left[s_{t+1}|s_t\right] = \mathbb{P}\left[s_{t+1}| s_1, \hdots, s_t\right]$. Similarly, $\mathcal{R}^{a_t}_{s_t,s_{t+1}} \in \mathbb{R}$ is the expected reward when taking action $a_t$ and transitioning from state $s_t$ to state $s_{t+1}$ (Eq. \ref{eq:mdp_reward_fun_1}), which is bounded, i.e. $|\mathcal{R}^{a_t}_{s_t,s_{t+1}}| \leq M < \infty$ for some $M \in \mathbb{R}$. If $\mathcal{P}^{a_t}_{s_t,s_{t+1}}$ and $\mathcal{R}^{a_t}_{s_t,s_{t+1}}$ are independent of the time-step $t$, the MDP is said to be \textit{stationary} and if $\mathcal{S}$ and $\mathcal{A}$ are finite, the MDP is said to be \textit{finite}. Finally, $\gamma \in \left(0,1\right]$ is the discount factor, $\rho_1 : \mathcal{S} \rightarrow [0,1]$ is the initial state distribution, and $T$ is the time horizon.
\begin{align}
&\mathcal{P}^{a_t}_{s_t,s_{t+1}} = \mathbb{P}\left[s_{t+1}| s_t, a_t\right]\label{eq:mdp_prob_1}\\
& \mathcal{R}^{a_t}_{s_t,s_{t+1}} = \mathbb{E}\left[r_{t+1}| a_t,  s_t, s_{t+1}\right] \label{eq:mdp_reward_fun_1}
\end{align}
The system evolves in discrete time-steps from $t=1$ to $t=T$, which constitute one \textit{episode} of the system.

A Partially Observed Markov Decision Process (POMDP) is an extension of an MDP \cite{howard_mdps,krishnamurthy_2016}. In contrast to an MDP, in a POMDP the states are not directly observable. A POMDP is defined by a nine-tuple $\mathcal{M}_{\mathcal{P}} = \langle \mathcal{S}, \mathcal{A}, \mathcal{P}^{a_t}_{s_t,s_{t+1}}, \mathcal{R}^{a_t}_{s_t,s_{t+1}},\gamma, \rho_1, T, \mathcal{O}, \mathcal{Z}\rangle$. The first seven elements define an MDP. $\mathcal{O}$ denotes the set of observations and $\mathcal{Z}(o_{t+1},s_{t+1},a_t) = \mathbb{P}[o_{t+1}|s_{t+1},a_t]$ is the observation function, where $o_{t+1} \in \mathcal{O}$, $s_{t+1} \in \mathcal{S}$, and $a_t \in \mathcal{A}$. If $\mathcal{O}$, $\mathcal{S}$, and $\mathcal{A}$ are finite, the POMDP is said to be finite.

The belief state $b_t \in \mathcal{B}$ is defined as $b_t(s)=\mathbb{P}[s_t=s|h_t]$ for all $s \in \mathcal{S}$. $b_t$ is a sufficient statistic of the state $s_t$ based on the history $h_t$ of the initial state distribution, the actions, and the observations: $h_t=(\rho_1,a_1,o_1,\hdots,a_{t-1},o_t) \in \mathcal{H}$. The belief space $\mathcal{B}=\Delta(\mathcal{S})$ is the unit $(|\mathcal{S}|-1)$-simplex \cite{pomdp_belief_optimal,ASTROM1965174}, where $\Delta(\mathcal{S})$ denotes the set of probability distributions over $\mathcal{S}$. By defining the state at time $t$ to be the belief state $b_t$, a POMDP can be formulated as a continuous-state MDP: $\mathcal{M} = \langle \mathcal{B}, \mathcal{A}, \mathcal{P}^{a_t}_{b_t,b_{t+1}}, \mathcal{R}_{b_t,b_{t+1}}^{a_t}, \gamma, \rho_1, T \rangle$.

The belief state can be computed recursively as follows \cite{krishnamurthy_2016}:
\begin{align}
b_{t+1}(s_{t+1}) &= C\mathcal{Z}(o_{t+1}, s_{t+1}, a_t)\sum_{s_t \in \mathcal{S}}\mathcal{P}^{a_t}_{s_t,s_{t+1}}b_t(s_t)\label{eq:belief_upd}
\end{align}
where $C=1/\sum_{s_{t+1} \in S}\mathcal{Z}(o_{t+1},s_{t+1},a_t) \sum_{s_t \in S}\mathcal{P}_{s_t,s_{t+1}}^{a_t}b_t(s_t)$ is a normalizing factor independent of $s_{t+1}$ to make $b_{t+1}$ sum to $1$.
\subsection{The Reinforcement Learning Problem}\label{sec:rl_prob2}
Reinforcement learning deals with the problem of choosing a sequence of actions for a sequentially observed state variable to maximize a reward function \cite{BertsekasTsitsiklis96,rl_bible}. This problem can be modeled with an MDP if the state space is observable, or with a POMDP if the state space is not fully observable.

In the context of an MDP, a policy is defined as a function $\pi: \{1,\hdots, T\} \times \mathcal{S} \rightarrow \Delta(\mathcal{A})$, where $\Delta(\mathcal{A})$ denotes the set of probability distributions over $\mathcal{A}$. In the case of a POMDP, a policy is defined as a function $\pi: \mathcal{H} \rightarrow \Delta(\mathcal{A})$, or, alternatively, as a function $\pi: \{1,\hdots, T\} \times \mathcal{B} \rightarrow \Delta(\mathcal{A})$. In both cases, a policy is called \textit{stationary} if it is independent of the time-step $t$ given the current state or belief state.

An optimal policy $\pi^{*}$ is a policy that maximizes the expected discounted cumulative reward over the time horizon:
\begin{align}
\pi^{*} &\in \argmax_{\pi \in \Pi} \mathbb{E}_{\pi}\left[\sum_{t=1}^{T}\gamma^{t-1}r_{t}\right] \label{eq:rl_prob}
\end{align}
where $\Pi$ is the policy space, $\gamma$ is the discount factor, $r_t$ is the reward at time $t$, and $\mathbb{E}_{\pi}$ denotes the expectation under $\pi$.

Optimal \textit{deterministic} policies exist for finite MDPs and POMDPs with bounded rewards and either finite horizons or infinite horizons with $\gamma \in (0,1)$ \cite{puterman,krishnamurthy_2016}. If the MDPs or POMDPs also are stationary and the horizons are either random or infinite with $\gamma \in (0,1)$, optimal \textit{stationary} policies exist \cite{puterman,krishnamurthy_2016}.

The Bellman equations relate any optimal policy $\pi^{*}$ to the two value functions $V^{*} : \mathcal{S} \rightarrow \mathbb{R}$ and $Q^{*}: \mathcal{S} \times \mathcal{A} \rightarrow \mathbb{R}$, where $\mathcal{S}$ and $\mathcal{A}$ are state and action spaces of an MDP \cite{bellman_eq}:
\begin{align}
V^{*}(s_t) &= \displaystyle\max_{a_t\in \mathcal{A}} \mathbb{E}\big[r_{t+1} + \gamma V^{*}(s_{t+1}) | s_t, a_t\big]\label{eq:bellman_eq_31} \\
Q^{*}(s_t,a_t) &= \mathbb{E}\big[r_{t+1} + \gamma V^{*}(s_{t+1}) | s_t, a_t\big] \label{eq:bellman_eq_33}\\
\pi^{*}(s_t) &\in \argmax_{a_t\in \mathcal{A}} Q^*(s_t,a_t)\label{eq:bellman_eq_34}
\end{align}
$V^{*}(s_t)$ and $Q^{*}(s_t,a_t)$ denote the expected cumulative discounted reward under $\pi^{*}$ for each state and state-action pair, respectively. Solving Eqs. \ref{eq:bellman_eq_31}-\ref{eq:bellman_eq_33} means computing the value functions from which an optimal policy can be obtained (Eq. \ref{eq:bellman_eq_34}). In the case of a POMDP, the Bellman equations contain $b_t$ instead of $s_t$ and $V^{*}(b_t)$ is convex \cite{smallwood_1}.

Two principal methods are used for finding an optimal policy in a finite MDP or POMDP: dynamic programming and reinforcement learning.

First, the dynamic programming method (e.g. value iteration \cite{bellman_eq,bert05,puterman}) assumes complete knowledge of the seven-tuple MDP or the nine-tuple POMDP and obtains an optimal policy by solving the Bellman equations iteratively (Eq. \ref{eq:bellman_eq_34}), with polynomial time-complexity per iteration for MDPs and PSPACE-complete time-complexity for POMDPs \cite{pspace_complexity}.

Second, the reinforcement learning method computes or approximates an optimal policy without requiring complete knowledge of the transition probabilities or observation probabilities of the MDP or POMDP. Three classes of reinforcement learning algorithms exist: \textit{value-based algorithms}, which approximate solutions to the Bellman equations (e.g. Q-learning \cite{watkins_thesis}); \textit{policy-based algorithms}, which directly search through policy space using gradient-based methods (e.g. Proximal Policy Optimization (PPO) \cite{ppo}); and \textit{model-based algorithms}, which learn the transition or observation probabilities of the MDP or POMDP (e.g. Dyna-Q \cite{rl_bible}). The three algorithm types can also be combined, e.g. through \textit{actor-critic} algorithms, which are mixtures of value-based and policy-based algorithms \cite{rl_bible}. In contrast to dynamic programming algorithms, reinforcement learning algorithms generally have no guarantees to converge to an optimal policy except for the tabular case \cite{jaakola_convergence_Q_NIPS1993, robbins_monro}.
\subsection{The Markovian Optimal Stopping Problem}\label{sec:optimal_stopping}
Optimal stopping is a classical problem domain with a well-developed theory \cite{wald,shirayev,stopping_book_1,chow1971great,bert05,ross_stochastic_dp,bather_decision_theory,puterman,poor_hadjiliadis_2008}. Example use cases for this theory include: asset selling \cite{bert05}, change detection \cite{tartakovsky_1}, machine replacement \cite{krishnamurthy_2016}, hypothesis testing \cite{wald}, gambling \cite{chow1971great}, selling decisions \cite{optimal_stopping_finance}, queue management \cite{roy_threshold}, advertisement scheduling \cite{optimal_multiple_stopping_social_media_1}, industrial control \cite{kalle_stopping_industrial}, and the secretary problem \cite{puterman,kleinberg_multi_secretary}.

Many variants of the optimal stopping problem have been studied. For example, discrete-time and continuous-time problems, finite horizon and infinite horizon problems, problems with fully observed and partially observed state spaces, problems with finite and infinite state spaces, Markovian and non-Markovian problems, and single-stop and multi-stop problems. Consequently, different solution methods for these variants have been developed. The most commonly used methods are the \textit{martingale approach} \cite{stopping_book_1,chow1971great,Snell1952TAMS} and the \textit{Markovian approach} \cite{shirayev,bert05,puterman,ross_stochastic_dp,bather_decision_theory}.

In this paper, we investigate the multiple stopping problem with $L$ stops, a finite time horizon $T$, discrete-time progression, bounded rewards, a finite state space, and the Markov property. We use the Markovian solution approach and model the problem as a POMDP, where the system state evolves as a discrete-time Markov process $(s_{t,l})_{t=1}^{T}$ that is partially observed and depends on the number of stops remaining $l \in \{1,\hdots,L\}$.

At each time-step $t$ of the decision process, two actions are available: ``stop'' ($S$) and ``continue'' ($C$). The \textit{stop} action with $l$ stops remaining yields a reward $\mathcal{R}^{S}_{s_t,s_{t+1},l_t}$ and if only one of the $L$ stops remain, the process terminates. In the case of a \textit{continue} action or a non-final stop action $a_t$, the decision process transitions to the next state according to the transition probabilities $\mathcal{P}^{a_t}_{s_t,s_{t+1},l_t}$ and yields a reward $\mathcal{R}^{a_t}_{s_t,s_{t+1},l_t}$.

The \textit{stopping time} with $l$ stops remaining is a random variable $\tau_l$ that is dependent on $s_1,\hdots,s_{\tau_l}$ and independent of $s_{\tau_{l}+1},\hdots s_{T}$ \cite{stopping_book_1}:
\begin{align}
\tau_l &= \inf\{t: t > \tau_{l+1}, a_t=S\}, & l\in 1,..,L,\text{ }\tau_{L+1}=0 \label{eq:stopping_time_def_1}
\end{align}

The objective is to find a stopping policy $\pi_l^{*}(s_t) \rightarrow \{S,C\}$ that depends on $l$ and maximizes the expected discounted cumulative reward of the stopping times $\tau_{L},\tau_{L-1},\hdots, \tau_1$:
\begin{align}
  &\pi_l^{*} \in \argmax_{\pi_l} \mathbb{E}_{\pi_l}\Bigg[\sum_{t=1}^{\tau_{L}-1}\gamma^{t-1}\mathcal{R}^{C}_{s_t,s_{t+1},L} + \gamma^{\tau_{L}-1}\mathcal{R}^{S}_{s_{\tau_L},s_{\tau_L+1},L}  \nonumber\\
  &+ \hdots + \sum_{t=\tau_2+1}^{\tau_{1}-1}\gamma^{t-1}\mathcal{R}^{C}_{s_t,s_{t+1},1} + \gamma^{\tau_{1}-1}\mathcal{R}^{S}_{s_{\tau_1},s_{\tau_1+1},1} \Bigg]\label{eq:optimal_stopping_2}
\end{align}
Due to the Markov property, any policy that satisfies Eq. \ref{eq:optimal_stopping_2} also satisfies the Bellman equation (Eq. \ref{eq:bellman_eq_34}), which in the partially observed case is:
\begin{align}
&\pi_l^{*}(b) \in \argmax_{\{S,C\}} \label{eq:optimal_stopping_1}\\
& \Bigg[\underbrace{\mathbb{E}_l\left[\mathcal{R}^{S}_{b,b^{o}_{S},l} + \gamma V_{l-1}^{*}(b^{o}_{S})\right]}_{\text{stop } (S)}, \underbrace{\mathbb{E}_l\left[\mathcal{R}^{C}_{b,b_C^o,l} + \gamma V_{l}^{*}(b_C^o)\right]}_{\text{continue } (C)}\Bigg]\text{\tiny }\forall b \in \mathcal{B} \nonumber
\end{align}
where $\pi_l$ is the stopping policy with $l$ stops remaining, $\mathbb{E}_l$ denotes the expectation with $l$ stops remaining, $b$ is the belief state, $V_{l}^{*}$ is the value function with $l$ stops remaining, $b^o_{S}$ and $b^{o}_{C}$ can be computed using Eq. \ref{eq:belief_upd}, and $\mathcal{R}_{b,b^o_{a},l}^{a}$ is the expected reward of action $a \in \{S,C\}$ in belief state $b_t$ when observing $o$ with $l$ stops remaining.

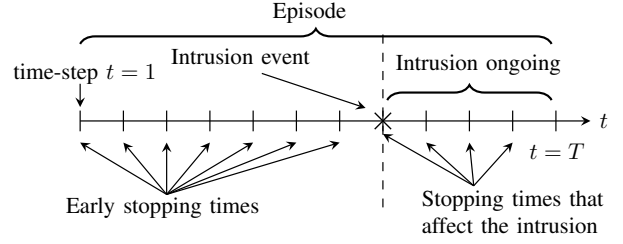
\begin{figure}
  \centering
  \scalebox{1.15}{
    \begin{tikzpicture}[fill=white, >=stealth,
    node distance=3cm,
    database/.style={
      cylinder,
      cylinder uses custom fill,
      shape border rotate=90,
      aspect=0.25,
      draw}]

    \tikzset{
node distance = 9em and 4em,
sloped,
   box/.style = {%
    shape=rectangle,
    rounded corners,
    draw=blue!40,
    fill=blue!15,
    align=center,
    font=\fontsize{12}{12}\selectfont},
 arrow/.style = {%
    line width=0.1mm,
    -{Triangle[length=5mm,width=2mm]},
    shorten >=1mm, shorten <=1mm,
    font=\fontsize{8}{8}\selectfont},
}

\draw[->, color=black] (0, 0) to (5.9, 0);
\draw[-, color=black] (0, -0.12) to (0, 0.12);
\draw[-, color=black] (0.5, -0.12) to (0.5, 0.12);
\draw[-, color=black] (1, -0.12) to (1, 0.12);
\draw[-, color=black] (1.5, -0.12) to (1.5, 0.12);
\draw[-, color=black] (2, -0.12) to (2, 0.12);
\draw[-, color=black] (2.5, -0.12) to (2.5, 0.12);
\draw[-, color=black] (3, -0.12) to (3, 0.12);
\draw[-, color=black] (3.5, -0.12) to (3.5, 0.12);
\draw[-, color=black] (4, -0.12) to (4, 0.12);
\draw[-, color=black] (4.5, -0.12) to (4.5, 0.12);
\draw[-, color=black] (5, -0.12) to (5, 0.12);
\draw[-, color=black] (5.5, -0.12) to (5.5, 0.12);

\draw[-, color=black, dashed] (3.5, 1) to (3.5, -1);

\node[inner sep=0pt,align=center, scale=1] (time) at (3.5,0)
{
\Crossk
};

\node[inner sep=0pt,align=center, scale=0.75] (time) at (1.9,0.75)
{
Intrusion event
};

\node[inner sep=0pt,align=center, scale=0.75] (time) at (0.1,0.55)
{
time-step $t=1$
};

\draw[->, color=black] (0, 0.43) to (0, 0.16);

\node[inner sep=0pt,align=center, scale=0.75] (time) at (4.65,0.65)
{
Intrusion ongoing
};

\node[inner sep=0pt,align=center, scale=0.75] (time) at (6.1,0)
{
$t$
};

\node[inner sep=0pt,align=center, scale=0.75] (time) at (5.56,-0.35)
{
$t=T$
};

\node[inner sep=0pt,align=center, scale=0.75] (time) at (1,-1)
{
Early stopping times
};
\draw[->, color=black] (1, -0.85) to (0.0, -0.25);
\draw[->, color=black] (1, -0.85) to (0.5, -0.25);
\draw[->, color=black] (1, -0.85) to (1, -0.25);
\draw[->, color=black] (1, -0.85) to (1.5, -0.25);
\draw[->, color=black] (1, -0.85) to (2, -0.25);
\draw[->, color=black] (1, -0.85) to (2.5, -0.25);
\draw[->, color=black] (1, -0.85) to (3, -0.25);

\node[inner sep=0pt,align=left, scale=0.75] (time) at (5,-1.05)
{
  Stopping times that\\
  affect the intrusion
};

\draw[->, color=black] (4.55, -0.75) to (3.5, -0.15);
\draw[->, color=black] (4.55, -0.75) to (4, -0.25);
\draw[->, color=black] (4.55, -0.75) to (4.5, -0.25);
\draw[->, color=black] (4.55, -0.75) to (5, -0.25);

\draw [decorate,decoration={brace,amplitude=5pt,mirror,raise=4pt},yshift=0pt,rotate=180, line width=0.25mm]
(-5.45,-0.1) -- (-3.55,-0.1) node [black,midway,xshift=0.1cm] {};

\draw[->, color=black] (2.1,0.55) to (3.3,0.15);

\draw [decorate,decoration={brace,amplitude=5pt,mirror,raise=4pt},yshift=0pt,rotate=180, line width=0.25mm]
(-5.45,-0.75) -- (0,-0.75) node [black,midway,xshift=0.1cm] {};

\node[inner sep=0pt,align=left, scale=0.75] (time) at (2.7,1.25)
{
Episode
};

\end{tikzpicture}
    }
    \caption{Optimal multiple stopping formulation of intrusion prevention; the horizontal axis represents time; $T$ is the time horizon; the episode length is $T-1$; the dashed line shows the intrusion start time; the optimal policy is to prevent the attacker at the time of intrusion.}
    \label{fig:stopping_times}
  \end{figure}
\section{Formalizing The Intrusion Prevention Use Case and Our Reinforcement Learning Approach}\label{sec:formal_model_2}
We first present a formal model of the use case described in Section \ref{sec:use_case} and then we introduce our solution method. Specifically, we first define a POMDP model of the intrusion prevention use case. Then, we apply the theory of dynamic programming and optimal stopping to obtain structural results of an optimal defender policy. Lastly, we describe our reinforcement learning approach to approximate an optimal policy.
\subsection{A POMDP Model of the Intrusion Prevention Use Case}\label{sec:pomdp_model}
We formulate the intrusion prevention use case as a multiple stopping problem where an intrusion starts at a random time and each stop is associated with a defensive action (Fig. \ref{fig:stopping_times}). We model this problem as a POMDP.

\subsubsection{Actions $\mathcal{A}$}\label{sec:actions}
The defender has two actions: ``stop'' ($S$) and ``continue'' ($C$). The action space is thus $\mathcal{A} = \{S,C\}$. We encode $S$ with $1$ and $C$ with $0$ to simplify the formal description below.

The number of stops that the defender must execute to prevent an intrusion is $L \geq 1$, which is a predefined parameter of our use case.
\subsubsection{States $\mathcal{S}$ and Initial State Distribution $\rho_1$}
The system state $s_t \in \{0,1\}$ is zero if no intrusion is occurring and $s_t=1$ if an intrusion is ongoing. In the initial state, no intrusion is occurring and $s_1 = 0$. Hence, the initial state distribution is the degenerate distribution $\rho_1(0) = 1$. Further, we introduce a terminal state $\emptyset \in \mathcal{S}$, which is reached after the defender takes the final stop action or after an intrusion is prevented (see below). The state space is thus $\mathcal{S} = \{0,1,\emptyset\}$.
\subsubsection{Observations $\mathcal{O}$}
The defender has a partial view of the system. If $s_t\neq \emptyset$, the defender observes $o_t = (l_t, \Delta x_t, \Delta y_t, \Delta z_t)$, where $l_t\in \{1,2,\hdots,L\}$ is the number of stops remaining and ($\Delta x_t$, $\Delta y_t$, $\Delta z_t$) are bounded counters that denote the number of severe IDS alerts, warning IDS alerts, and login attempts generated during time-step $t$, respectively. If the system is in the terminal state, the defender observes $o_{T}=\emptyset$. Hence, the observation space is $\mathcal{O} = \{0,\hdots,\Delta x_{max}\} \times \{0,\hdots,\Delta y_{max}\} \times \{0,\hdots,\Delta z_{max}\} \cup \text{ }\emptyset$.
\subsubsection{Transition Probabilities $\mathcal{P}^{a_t}_{s_t,s_{t+1},l_t}$}
We model the start of an intrusion by a Bernoulli process $(Q_t)_{t=1}^{T}$, where $Q_t \sim Ber(p=0.01)$ is a Bernoulli random variable. The time of the first occurrence of $Q_t=1$ is the start time of the intrusion $I_t$, which thus is geometrically distributed, i.e., $I_t \sim Ge(p=0.01)$ (Fig. \ref{fig:geo_only_2}).

\begin{figure}
  \centering
    \scalebox{0.45}{
      \includegraphics{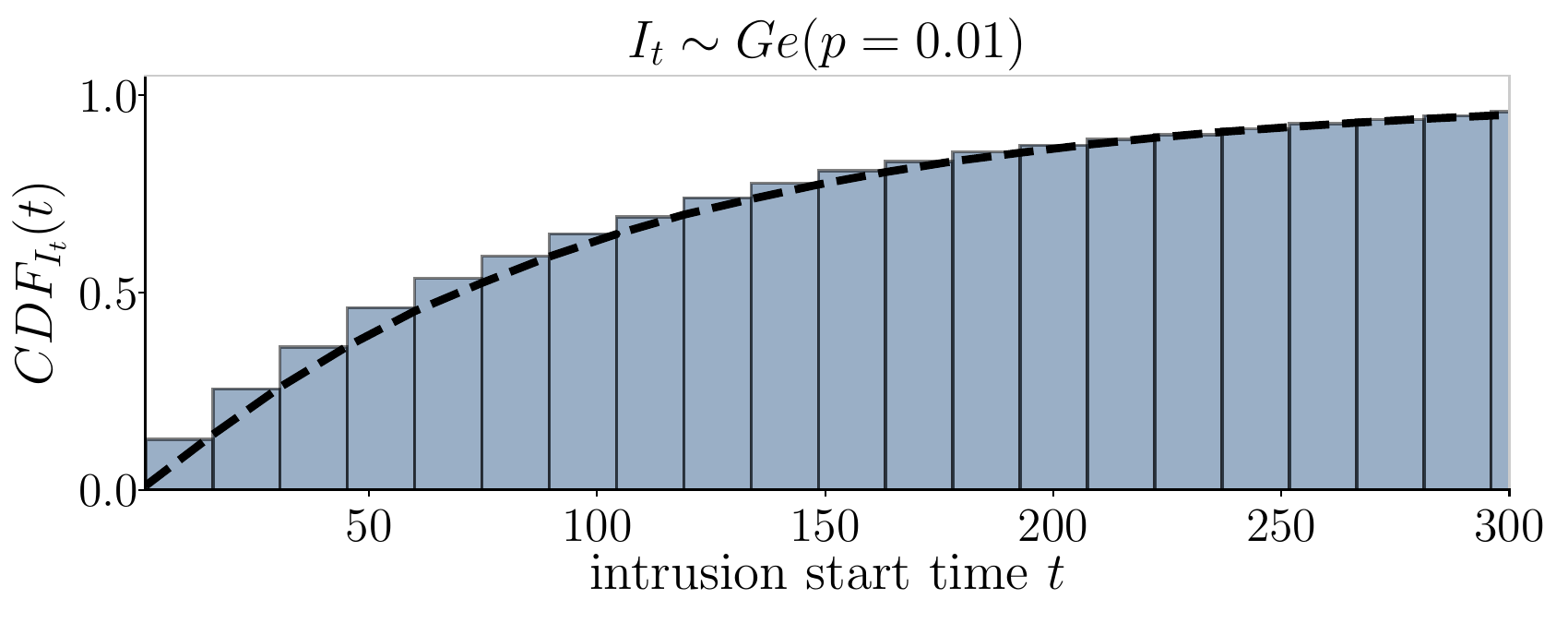}
    }
    \caption{The cumulative distribution function (CDF) of the intrusion start time $I_t$.}
    \label{fig:geo_only_2}
\end{figure}

We define the time-homogeneous transition probabilities $\mathcal{P}^{a_t}_{s_t,s_{t+1},l_t} =\mathbb{P}_{l_t}\left[s_{t+1}| s_t, a_t\right]$ as follows:
\begin{align}
&\mathbb{P}_{1}\left[\emptyset \middle| \cdot , 1\right] = \mathbb{P}_{l_t}\left[\emptyset \middle| \emptyset,\cdot\right]=1 \label{eq:tp_1}\\
&\mathbb{P}_{l_t}\left[0 \middle|0, a_t \right] = 1-p \quad\quad\quad\quad\quad\quad\quad \text{ if $l_t-a_t > 0$}\label{eq:tp_2}\\
&\mathbb{P}_{l_t}\left[1 \middle|0,a_t \right]= p \quad\quad\quad\quad\quad\quad\quad\quad\quad \text{ if $l_t-a_t > 0$}\label{eq:tp_3}\\
&\mathbb{P}_{l_t}\left[1 \middle|1, a_t \right] = 1 \quad\quad\quad\quad\quad\quad\quad\quad\quad \text{ if $l_t-a_t > 0$}\label{eq:tp_4}
\end{align}
where $\mathbb{P}_{l_t}$ denotes the probability with $l_t$ stops remaining. All other state transitions occur with probability $0$.

Eq. \ref{eq:tp_1} defines the transition probabilities to the terminal state $\emptyset$. The terminal state is reached when the \textit{final} ($l_t=1$) stop action $S$ ($a_t = 1$) is taken. If Eq. \ref{eq:tp_1} is not applicable, i.e., if the system does not reach the terminal state, then the transition probabilities when taking action $S$ ($a_t=1$) or $C$ ($a_t=0$) are defined by Eqs. \ref{eq:tp_2}-\ref{eq:tp_4}.

Eq. \ref{eq:tp_2} captures the case where no intrusion occurs and $s_{t+1}=s_t=0$; Eq. \ref{eq:tp_3} specifies the case when the intrusion starts where $s_t=0$ and $s_{t+1}=1$; and Eq. \ref{eq:tp_4} describes the case where an intrusion is in progress and $s_{t+1}=s_t = 1$.

With this definition of the transition probabilities, the evolution of the system can be understood using the state transition diagram in Fig. \ref{fig:state_transitions}.
\subsubsection{Observation Function $\mathcal{Z}(o_{t+1},s_{t+1},a_t)$}\label{sec:obs_fun}
We assume that the number of IDS alerts and login attempts generated during one time-step are discrete random variables $X \sim f_X$, $Y \sim f_Y$, $Z \sim f_Z$ that depend on the state. Consequently, the probability that $\Delta x$ severe alerts, $\Delta y$ warning alerts, and $\Delta z$ login attempts occur during time-step $t$ can be expressed as $f_{XYZ}(\Delta x,\Delta y,\Delta z|s_t)$.

We define the time-homogeneous observation function $\mathcal{Z}(o_{t+1},s_{t+1},a_t) =\mathbb{P}[o_{t+1}|s_{t+1},a_t]$ as follows:
\begin{align}
&\mathcal{Z}\big((l_t,\Delta x,\Delta y,\Delta z),s_t,\cdot \big) = f_{XYZ}(\Delta x,\Delta y,\Delta z |s_t) \label{eq:obs_function}\\
&\mathcal{Z}\big(\emptyset,\emptyset,\cdot\big) = 1
\end{align}\normalsize
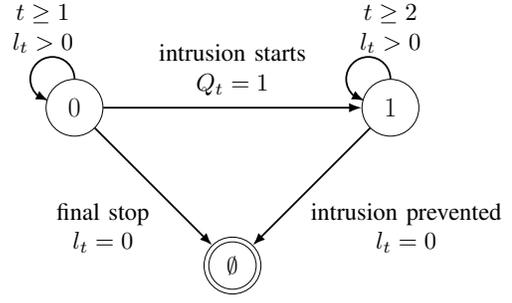
\begin{figure}
  \centering
  \scalebox{1.05}{
          \begin{tikzpicture}[fill=white, >=stealth,
    node distance=3cm,
    database/.style={
      cylinder,
      cylinder uses custom fill,
      shape border rotate=90,
      aspect=0.25,
      draw}]

    \tikzset{
node distance = 9em and 4em,
sloped,
   box/.style = {%
    shape=rectangle,
    rounded corners,
    draw=blue!40,
    fill=blue!15,
    align=center,
    font=\fontsize{12}{12}\selectfont},
 arrow/.style = {%
    line width=0.1mm,
    -{Triangle[length=5mm,width=2mm]},
    shorten >=1mm, shorten <=1mm,
    font=\fontsize{8}{8}\selectfont},
}

\node[scale=0.8] (kth_cr) at (0,2.15)
{
  \begin{tikzpicture}
\node[scale=1] (level1) at (-1.7,-5.6)
{
  \begin{tikzpicture}
\node[draw,circle, minimum width=15mm, scale=0.6](s0) at (0,0) {\huge $0$};
\node[draw,circle, minimum width=15mm, scale=0.6](s1) at (5,0) {\huge $1$};
\node[draw,circle, minimum width=15mm, scale=0.6](s2) at (2.5,-2.5) {};
\node[draw,circle, minimum width=15mm, scale=0.5](s4) at (2.5,-2.5) {\huge$\emptyset$};
\node[inner sep=0pt,align=center, scale=1.1] (time) at (-0.5,1.25)
{
$t\geq 1$\\
$l_t > 0$
};

\node[inner sep=0pt,align=center, scale=1.1] (time) at (5,1.25)
{
$t\geq 2$\\
$l_t>0$
};

\node[inner sep=0pt,align=center, scale=1.1] (time) at (2.5,0.6)
{
intrusion starts\\
$Q_t=1$
};
\node[inner sep=0pt,align=center, scale=1.1,rotate=0] (time) at (0.45,-1.9)
{
final stop\\
$l_t=0$
};
\node[inner sep=0pt,align=center, scale=1.1,rotate=0] (time) at (5.25,-1.9)
{
intrusion prevented\\
$l_t=0$
};
\draw[thick,-{Latex[length=2mm]}] (s0) to (s1);
\draw[thick,-{Latex[length=2mm]}] (s1) to (s2);
\draw[thick,-{Latex[length=2mm]}] (s0) to (s2);
\draw[thick,-{Latex[length=2mm]}] (s0.90) arc (0:260:3.5mm);
\draw[thick,-{Latex[length=2mm]}] (s1.90) arc (0:260:3.5mm);
    \end{tikzpicture}
  };
    \end{tikzpicture}
  };

\end{tikzpicture}
    }
    \caption{State transition diagram of the POMDP: each circle represents a state; an arrow represents a state transition; a label indicates the event that triggers the state transition; an episode starts in state $s_1=0$ with $l_1=L$.}
    \label{fig:state_transitions}
  \end{figure}
\subsubsection{Reward Function $\mathcal{R}_{s_t,l_t}^{a_t}$}\label{sec:reward_fun}
The objective of the intrusion prevention use case is to maintain service on the infrastructure while, at the same time, preventing a possible intrusion. Therefore, we define the reward function to give the maximal reward if the defender maintains service until the intrusion starts and then prevents the intrusion by taking $L$ stop actions.

The reward per time-step $\mathcal{R}^{a_t}_{s_t,l_t}$ is parameterized by the reward that the defender receives for stopping an intrusion ($R_{st}=50$), the reward for maintaining service ($R_{sla}=1$), and the loss of being intruded ($R_{int}=-10$):
\begin{align}
&\mathcal{R}^{\cdot}_{\emptyset,0} = 0 \label{eq:reward_0}\\
&\mathcal{R}^{S}_{s_t,l_t} = s_tR_{st}/4l_t && s_t \in \{0,1\} \label{eq:reward_5}\\
&\mathcal{R}^{C}_{s_t,l_t} = R_{sla} + s_tR_{int}/L && s_t \in \{0,1\}\label{eq:reward_3}
\end{align}\normalsize
Eq. \ref{eq:reward_0} states that the reward in the terminal state is zero. Eq. \ref{eq:reward_5} indicates that each stop incurs a cost by interrupting service and possibly a reward if it affects an ongoing intrusion. Lastly, Eq. \ref{eq:reward_3} states that the defender receives a positive reward for maintaining service and a loss for each time-step that it is under intrusion. (Remark: the reward function can equivalently be stated to give a cumulative reward upon transitioning to the terminal state and zero reward otherwise \cite{puterman}.)
\subsubsection{Time Horizon $T_{\emptyset}$}\label{sec:time_horizon}
The time horizon $T_{\emptyset}$ is a random variable that indicates the time $t$ when the terminal state $\emptyset$ is reached. Since the expected time of intrusion $\mathbb{E}[I_t]$ is finite, it follows that $\mathbb{E}_{\pi_l}\left[T_{\emptyset}\right] < \infty$ for any policy $\pi_l$ that is guaranteed to use $L$ stops as $t\rightarrow \infty$. Further, as the continue reward is negative when $t > I_t$, the optimal stopping times $\tau_1, \hdots, \tau_L$ exist. (Remark: it is also possible to define $T=\infty$ and let $\emptyset$ be an infinitely absorbing state.)
\subsubsection{Policy Space $\Pi_{l}$ and Objective Function $J$}
As the POMDP is stationary and the time horizon $T_{\emptyset}$ is not predetermined, it is sufficient to consider stationary policies. Further, since the POMDP is finite, an optimal deterministic policy exists \cite{puterman,krishnamurthy_2016}. Despite this, we consider stochastic policies to enable smooth optimization. Specifically, we consider the space of stationary stochastic policies $\Pi_{l}$ where $\pi_{l} \in \Pi_{l}$ is a policy $\pi_l: \mathcal{B} \rightarrow \Delta(\mathcal{A})$, which depends on $l \in \{1,\hdots,L\}$.

An \textit{optimal} policy $\pi^{*}_{l} \in \Pi_{l}$ maximizes the expected discounted cumulative reward over the horizon $T_{\emptyset}$:
\begin{align}
&J(\pi_{l}) = \mathbb{E}_{\pi_{l}}\left[\sum_{t=1}^{T_{\emptyset}}\gamma^{t-1}\mathcal{R}^{a_t}_{s_t,l_t} \right]\label{eq:rl_prob2}\\
&\pi^{*}_{l} \in \argmax_{\pi_{l} \in \Pi_{l}}J(\pi_{l})\label{eq:rl_prob3}
\end{align}
We set the discount factor to be $\gamma=1$. (The objective in Eq. \ref{eq:rl_prob2} is upper bounded when $\gamma=1$ since $\mathbb{E}_{\pi_{l}}[T_{\emptyset}]$ is finite for any policy $\pi_{l} \in \Pi_{l}$ that is guaranteed to use $L$ stops as $t\rightarrow \infty$, which is true for any optimal policy (see Lemma \ref{lemma:stops_required} in Appendix \ref{appendix:proof_structural_result_2}).)

Eq. \ref{eq:rl_prob2} defines an optimization problem which reflects the objective of our use case. In the following section, we state structural properties of an optimal policy that solves this problem.
\subsection{Threshold Properties of an Optimal Policy}\label{sec:dp_opt}
A policy that solves the multiple stopping problem is a solution to Eqs. \ref{eq:rl_prob2}-\ref{eq:rl_prob3}. We know from the theory of dynamic programming that this policy satisfies the Bellman equation formulated in terms of the belief state (Eq. \ref{eq:optimal_stopping_1}) \cite{krishnamurthy_2016,pomdp_belief_optimal}.

The belief state $b_t$ is defined as $b_t(s_t)=\mathbb{P}[s_t|h_t]$ (see Section \ref{sec:mdps}). As the state space of the POMDP is $\mathcal{S} = \{0,1,\emptyset\}$ (see Fig. \ref{fig:state_transitions}), $b_t$ is a probability vector with two components: $b_{t}(0) = \mathbb{P}[s_t=0|h_t]$ and $b_{t}(1) = \mathbb{P}[s_t=1|h_t]$, where $t=1, \hdots T_{\emptyset}-1$. Further, since $b_t(0) = 1-b_t(1)$, the belief state is determined by $b_t(1)$ and the \textit{belief space} $\mathcal{B}$ can be described by the unit interval, i.e. $\mathcal{B} = [0,1]$.

We partition $\mathcal{B}$ into two sets\textemdash the stopping set $\mathscr{S}_{l} = \{b(1) \in [0,1] : \pi_l^{*}\big(b(1)\big) = S\}$, which contains the belief states where it is optimal to \textit{stop}, and the continuation set $\mathscr{C}_{l} = \{b(1) \in [0,1] : \pi_l^{*}\big(b(1)\big) = C\}$, which contains the belief states where it is optimal to \textit{continue}. The number of stops remaining, $l$, ranges from $1$ to $L$.

Applying the theory developed in \cite{krishnamurthy_2016,Nakai1985,optimal_multiple_stopping_social_media_1}, we obtain the following structural result for an optimal policy.
\begin{theorem}\label{thm:structural_result_2}
Given the POMDP in Section \ref{sec:pomdp_model}, let $L$ denote the number of stop actions, $f_{XYZ|s}$ the conditional distribution of the observations, $b(1)$ the belief state, $\mathscr{S}_{l}$ the stopping set, and $\mathscr{C}_{l}$ the continuation set. The following holds:
\begin{enumerate}[(A)]
\item
  \begin{align}
\mathscr{S}_{l-1} \subseteq \mathscr{S}_{l} && l \in \{1,\hdots L\} \label{eq:thm_1_0}
\end{align}
\item If $L=1$, there exists a value $\alpha^{*} \in [0,1]$ and an optimal policy $\pi_L^{*}$ that satisfies:
\begin{align}
\pi_L^{*}(b(1)) = S \iff b(1) \geq \alpha^{*} \label{eq:thm_1_1}
\end{align}
\item If $L \geq 1$ and $f_{XYZ|s}$ is totally positive of order 2 (i.e., TP2), there exist $L$ values $\alpha^{*}_{1} \geq \alpha^{*}_{2} \geq \hdots \geq \alpha^{*}_L \in [0,1]$ and an optimal policy $\pi_l^{*}$ that satisfies:
\begin{align}
\pi_l^{*}(b(1)) = S \iff b(1) \geq \alpha_l^{*} \quad l\in \{1,\hdots,L\} \label{eq:thm_1_2}
\end{align}
\end{enumerate}
\end{theorem}
\begin{proof}[Proof.]
See Appendix \ref{appendix:proof_structural_result_2}.
\end{proof}
Theorem \ref{thm:structural_result_2}.A states that the stopping sets have a nested structure. This means that if it is optimal to stop when $b(1)$ has a certain value while $l-1$ stops remain, then it is also optimal to stop for the same value when $l$ or more stops remain.

Theorem \ref{thm:structural_result_2}.B and Theorem \ref{thm:structural_result_2}.C state that there exist an optimal policy with threshold properties (see Fig. \ref{fig:threshold_policy_3}). If $L \geq 1$, an additional condition applies: the probability matrix of $f_{XYZ|s}$ must be TP2 (all second order minors must be non-negative) \cite[Definition 10.2.1, pp. 223]{krishnamurthy_2016}\cite{Karlin1964TotalPA}. This condition is satisfied for example if $f_{XYZ|s}$ is stochastically monotone in $s$.

Knowing that there exists optimal policies with special structure has two benefits. First, insight into the structure of optimal policies often leads to a concise formulation and efficient implementation of the policies \cite{puterman, serkan_gyorgy_game}. This is obvious in the case of threshold policies. Second, the complexity of computing or learning an optimal policy can be reduced by exploiting structural properties \cite{krishnamurthy_2016, roy_threshold}. In the following section, we describe a reinforcement learning algorithm that exploits the structural result in Theorem \ref{thm:structural_result_2}.
\begin{figure}
  \centering
  \scalebox{1.05}{
    \begin{tikzpicture}[fill=white, >=stealth,
    node distance=3cm,
    database/.style={
      cylinder,
      cylinder uses custom fill,
      shape border rotate=90,
      aspect=0.25,
      draw}]

    \tikzset{
node distance = 9em and 4em,
sloped,
   box/.style = {%
    shape=rectangle,
    rounded corners,
    draw=blue!40,
    fill=blue!15,
    align=center,
    font=\fontsize{12}{12}\selectfont},
 arrow/.style = {%
    line width=0.1mm,
    -{Triangle[length=5mm,width=2mm]},
    shorten >=1mm, shorten <=1mm,
    font=\fontsize{8}{8}\selectfont},
}

\node[scale=1] (system) at (0,0)
{
\begin{tikzpicture}
\draw[->, color=black] (0.0,0) to (6,0);

\node[inner sep=0pt,align=center, scale=0.8] (time) at (6.3,0)
{
  $b(1)$
};

\node[inner sep=0pt,align=center, scale=0.8] (time) at (0.05,-0.3)
{
$0$
};

\node[inner sep=0pt,align=center, scale=0.8] (time) at (5.75,-0.3)
{
$1$
};

\draw[-, color=black] (5.7,0.1) to (5.7,-0.1);

\draw[-, color=black] (0,0.1) to (0,-0.1);


\draw [decorate,decoration={brace,amplitude=5pt,mirror,raise=4pt},yshift=0pt,rotate=180, line width=0.20mm]
(-5.65,0.1) -- (-4.35,0.1) node [black,midway,xshift=0.1cm] {};


\node[inner sep=0pt,align=center, scale=0.8] (time) at (5.1,0.4)
{
$\mathscr{S}_{1}$
};

\draw [decorate,decoration={brace,amplitude=5pt,mirror,raise=4pt},yshift=0pt,rotate=180, line width=0.20mm]
(-5.65,-0.4) -- (-3.5,-0.4) node [black,midway,xshift=0.1cm] {};

\node[inner sep=0pt,align=center, scale=0.8] (time) at (4.8,0.9)
{
$\mathscr{S}_{2}$
};

\node[inner sep=0pt,align=center, scale=0.8] (time) at (4.1,1.25)
{
$\vdots$
};

\draw [decorate,decoration={brace,amplitude=5pt,mirror,raise=4pt},yshift=0pt,rotate=180, line width=0.20mm]
(-5.65,-1.18) -- (-2.3,-1.18) node [black,midway,xshift=0.1cm] {};

\node[inner sep=0pt,align=center, scale=0.8] (time) at (4.15,1.7)
{
$\mathscr{S}_{L}$
};

\draw[-, color=black] (4.3,0.1) to (4.3,-0.1);

\draw[-, color=black] (3.48,0.1) to (3.48,-0.1);

\draw[-, color=black] (2.35,0.1) to (2.35,-0.1);

\node[inner sep=0pt,align=center, scale=0.8] (time) at (4.3,-0.3)
{
$\alpha^{*}_{1}$
};
\node[inner sep=0pt,align=center, scale=0.8] (time) at (3.48,-0.3)
{
$\alpha^{*}_{2}$
};

\node[inner sep=0pt,align=center, scale=0.8] (time) at (2.35,-0.3)
{
$\alpha^{*}_{L}$
};

\node[inner sep=0pt,align=center, scale=0.8] (time) at (2.8,-0.3)
{
$\hdots$
};







\end{tikzpicture}
};

\end{tikzpicture}
  }
  \caption{Illustration of Theorem \ref{thm:structural_result_2}: there exist $L$ thresholds $\alpha^{*}_{1} \geq \alpha^{*}_{2}, \hdots, \geq \alpha^{*}_{L} \in \mathcal{B}$ and an optimal threshold policy $\pi^{*}_l$ that satisfies Eqs. \ref{eq:thm_1_0}-\ref{eq:thm_1_2}.}
    \label{fig:threshold_policy_3}
  \end{figure}
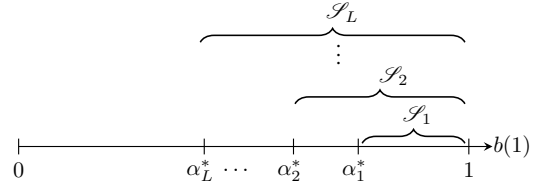

\subsection{Our Reinforcement Learning Algorithm: \textsc{T-SPSA}}\label{sec:rl_approach}
Theorem \ref{thm:structural_result_2} states that under given assumptions and given $L \geq 1$ stop actions, there exists an optimal policy which uses $L$ thresholds $\alpha^{*}_1 \geq \alpha^{*}_2, \hdots, \geq \alpha^{*}_L \in [0,1]$. We present an algorithm, which we call \textsc{T-SPSA}, that computes these thresholds through reinforcement learning.

We parameterize $\pi$ with a vector $\theta \in \mathbb{R}^{L}$. The component $\theta_l$ of $\theta$ relates to the threshold with $l \in \{1, \hdots L\}$ stops remaining. \textsc{T-SPSA} updates $\theta$ through stochastic gradient ascent with the following gradient \cite{policy_gradient_sutton}:
\begin{align}
\nabla_{\theta} J(\theta) &= \mathbb{E}_{\pi_{\theta,l}}\Bigg[\sum_{t=1}^{T_{\emptyset}}\nabla_{\theta}\log\pi_{\theta,l}(a_t|b_t(1))\sum_{\tau=t}^{T_{\emptyset}}\mathcal{R}^{a_{\tau}}_{b_{\tau},l_t}\Bigg] \label{eq:pg_objective}
\end{align}
To ensure differentiability, we define $\pi_{\theta,l}$ to be a smooth stochastic policy that approximates a threshold policy:
\begin{align}
\pi_{\theta,l}\big(S|b(1)\big) = \left(1 + \left(\frac{b(1)(1-\sigma(\theta_l))}{\sigma(\theta_l)(1-b(1))}\right)^{-20}\right)^{-1} \label{eq:smooth_threshold}
\end{align}
where $\sigma(\cdot)$ is the sigmoid function and $\sigma(\theta_{1}), \sigma(\theta_{2}), \hdots, \sigma(\theta_{L}) \in [0,1]$ are the $L$ thresholds.

We learn the threshold vector $\theta$ through simulation of the POMDP as follows. First, we initialize $\theta_{(1)} \in \mathbb{R}^L$ randomly. Second, for each iteration $n \in \{1,2,\hdots \}$ of \textsc{T-SPSA}, we perturb $\theta_{(n)}$ to obtain $\theta_{(n)} + c_n\Delta_n$ and $\theta_{(n)} - c_n\Delta_n$, where $c_n \in \mathbb{R}$ and $\Delta_n \in \mathbb{R}^L$. Then, we run two POMDP episodes where the defender takes actions according to the two perturbed threshold vectors (Eq. \ref{eq:smooth_threshold}). We then use the obtained episode outcomes $\hat{J}(\theta_{(n)} + c_n\Delta_n)$ and $\hat{J}(\theta_{(n)} - c_n\Delta_n)$ to estimate the gradient in Eq. \ref{eq:pg_objective} using the Simultaneous Perturbation Stochastic Approximation (SPSA) gradient estimator \cite{spsa, spsa_impl}:
\begin{align}
\left(\hat{\nabla}_{\theta_{(n)}}J(\theta_{(n)})\right)_{i} &= \frac{\hat{J}(\theta_{(n)} + c_n\Delta_n) - \hat{J}(_{(n)} - c_n\Delta_n)}{2c_n(\Delta_n)_{i}}
\end{align}
where $i\in\{1,\hdots,L\}$ is the component index of the gradient, $c_n = \frac{c}{n^{\lambda}}$ is the perturbation size and $c$ and $\lambda$ are hyperparameters.

The perturbation vector $\Delta_n$ is defined as:
\begin{align}
  (\Delta_n)_i =
  \begin{dcases}
 +1 & \quad \text{with probability $\frac{1}{2}$} \\
  -1 & \quad \text{with probability $\frac{1}{2}$}
\end{dcases}
\end{align}
Next, we use the estimated gradient and the stochastic approximation algorithm \cite{robbins_monro} to update the vector of thresholds to maximize $J(\theta)$ (Eq. \ref{eq:rl_prob2}):
\begin{align}
\theta_{(n+1)} &= \theta_{(n)} + a_n\hat{\nabla}_{\theta_{(n)}}J(\theta_{(n)})
\end{align}
where $a_n = \frac{a}{(n + A)^{\epsilon}}$ is the step size and $A$ and $\epsilon$ are hyperparameters.

This process of running two episodes and updating the threshold vector continues until it has sufficiently converged. The described algorithm, \textsc{T-SPSA}, converges to a \textit{local} maximum of $J(\theta)$ with probability one under standard conditions \cite{spsa}. For this reason, we run the algorithm several times with different initial conditions. We list the pseudocode of \textsc{T-SPSA} in Appendix \ref{appendix:spsa_thresholds} and give its hyperparameters in Appendix \ref{appendix:hyperparameters}. Our Python implementation of \textsc{T-SPSA} is available at: \cite{github_cnsm_21_hammar_stadler}.

\begin{table}
\centering
\begin{tabular}{lll} \toprule
  {\textit{Client}} & {\textit{Functions}} & {\textit{Application servers}} \\ \midrule
  $1$ & HTTP, SSH, SNMP, ICMP & $N_2,N_3,N_{10},N_{12}$\\
  $2$ & IRC, PostgreSQL, SNMP & $N_{31},N_{13},N_{14},N_{15},N_{16}$\\
  $3$ & FTP, DNS, Telnet & $N_{10}, N_{22}, N_{4}$ \\
  \bottomrule\\
\end{tabular}
\caption{Emulated client population; each client interacts with application servers using a set of network functions.}\label{tab:client_profiles}
\end{table}
\begin{table}
  \centering
\resizebox{0.95\columnwidth}{!}{%
\begin{tabular}{lll} \toprule
  {\textit{$L-l_t$}} & {\textit{Action}} & {\textit{Command in the Emulation}} \\ \midrule
  $0$ & Revoke user certificate & \texttt{openssl ca -revoke <certificate>}\\
  $1$ & Blacklist IPs & \texttt{iptables -A INPUT -s <ip> -j DROP}\\
  $2$ & Block gateway & \texttt{iptables -A INPUT -i eth0 -j DROP}\\
  \bottomrule\\
\end{tabular}
}
\caption{Defender stop commands in the emulation.}\label{tab:defender_stop_actions}
\end{table}
\begin{table*}
\centering
\resizebox{1\textwidth}{!}{%
\begin{tabular}{llll} \toprule
  {\textit{Time-steps $t$}} & {\textsc{NoviceAttacker}} & {\textsc{ExperiencedAttacker}} & {\textsc{ExpertAttacker}} \\ \midrule
  $1$-$I_t\sim Ge(0.01)$ & (Intrusion has not started) & (Intrusion has not started) & (Intrusion has not started)\\
  $I_t+1$-$I_t+6$ & $\text{\textsc{Recon}}_1$, brute-force attacks (SSH,Telnet,FTP) & $\text{\textsc{Recon}}_2$, CVE-2017-7494 exploit on $N_4$, & $\text{\textsc{Recon}}_3$, CVE-2017-7494 exploit on $N_4$,\\
                            & on $N_2,N_{4},N_{10}$, login($N_2,N_4,N_{10}$), & brute-force attack (SSH) on $N_2$, login($N_2,N_4$),& login($N_4$), backdoor($N_4$) \\
                            & backdoor($N_2,N_4,N_{10}$) & backdoor($N_2,N_4$), $\text{\textsc{Recon}}_2$ & $\text{\textsc{Recon}}_3$, SQL Injection on $N_{18}$ \\
  $I_t+7$-$I_t+10$ & $\text{\textsc{Recon}}_1$, CVE-2014-6271 on $N_{17}$, & CVE-2014-6271 on $N_{17}$, login($N_{17}$) & login($N_{18}$), backdoor($N_{18}$), \\
  & login($N_{17}$), backdoor($N_{17}$) & backdoor($N_{17}$), SSH brute-force attack on $N_{12}$  & $\text{\textsc{Recon}}_3$, CVE-2015-1427 on $N_{25}$ \\
  $I_t+11$-$I_t+14$ & SSH brute-force attack on $N_{12}$, login($N_{12}$) & login($N_{12}$), CVE-2010-0426 exploit on $N_{12}$, & login($N_{25}$), backdoor($N_{25}$),\\
  & CVE-2010-0426 exploit on $N_{12}$, $\text{\textsc{Recon}}_1$ & $\text{\textsc{Recon}}_2$, SQL Injection on $N_{18}$ & $\text{\textsc{Recon}}_3$, CVE-2017-7494 exploit on $N_{27}$\\
  $I_t+15$-$I_t+16$ & & login($N_{18}$), backdoor($N_{18}$) & login($N_{27}$), backdoor($N_{27}$) \\
  $I_t+17$-$I_t+19$ & & $\text{\textsc{Recon}}_2$, CVE-2015-1427 on $N_{25}$, login($N_{25}$) &  \\
  \bottomrule\\
\end{tabular}
}
\caption{Attacker actions in the emulation.}\label{tab:static_attackers}
\end{table*}
\section{Emulating the Target Infrastructure to Instantiate the Simulation and to Evaluate the Learned Policies}\label{sec:policy_learning_results}
To simulate episodes of the POMDP and to compute the belief state we must know the distributions of alerts and login attempts conditioned on the system state. We estimate these distributions using measurements from the emulation system shown in Fig. \ref{fig:method}. Moreover, to evaluate the performance of policies learned in the simulation system, we run episodes in the emulation system by executing actions of an emulated attacker and having the defender execute stop actions at times given by the learned policies.
\subsection{Emulating the Target Infrastructure}\label{sec:emu_target_inf}
The emulation system executes on a cluster of machines that runs a virtualization layer provided by Docker \cite{docker} containers and virtual links. It implements network isolation and traffic shaping on the containers using network namespaces and the NetEm module in the Linux kernel \cite{netem}. Resource constraints of the containers, e.g. CPU and memory constraints, are enforced using cgroups.

The configuration of the emulated infrastructure is given by the topology in Fig. \ref{fig:system2} and the configuration in Appendix \ref{appendix:infrastructure_configuration}. The system emulates the clients, the attacker, the defender, as well as $31$ physical components of the target infrastructure (e.g application servers and the gateway). Physical entities are emulated and software functions are executed in Docker containers of the emulation system. The software functions replicate important components of the target infrastructure, such as, web servers, databases, and an IDS.

We emulate internal connections between servers in the infrastructure as full-duplex loss-less connections with bit capacities of $1000$ Mbit/s in both directions and emulate external connections between the gateway and the client population and the attacker as full-duplex connections with bit capacities of $100$ Mbit/s with $0.1\%$ packet loss in normal operation and random bursts of $1\%$ packet loss.

The \textit{client population} is emulated by three Docker containers that interact with the application servers through functions and protocols listed in Table \ref{tab:client_profiles}.

The emulation evolves in time-steps. During each step, the defender and the attacker can perform one action each. The \textit{defender} executes either a continue action or a stop action. The continue action has no effect on the progression of the emulation but the stop action has. We have implemented $L=3$ stop actions which are listed in Table \ref{tab:defender_stop_actions}. The \textit{first} stop revokes all user certificates and recovers user accounts compromised by the attacker. The \textit{second} and \textit{third} stops update the firewall configuration of the gateway. Specifically, the \textit{second} stop adds a rule to the firewall that drops incoming traffic from IP addresses that have been flagged by the IDS and the \textit{third} stop blocks \textit{all} incoming traffic.
\begin{table}
\centering
\begin{tabular}{lll} \toprule
  {\textit{Attacker}} & {$L$} & {Reconnaissance} \\ \midrule
  \textsc{NoviceAttacker} & $1$ & TCP/UDP scan\\
  \textsc{ExperiencedAttacker} & $2$ & ICMP ping scan\\
  \textsc{ExpertAttacker} & $3$ & ICMP ping scan \\
  \bottomrule\\
\end{tabular}
\caption{Number of stops required to prevent the attacker $L$ and reconnaissance commands of the attacker profiles.}\label{tab:static_attackers_differences}
\end{table}

We have implemented three \textit{attacker profiles}: \textsc{NoviceAttacker}, \textsc{ExperiencedAttacker}, and \textsc{ExpertAttacker}, all of which execute the sequence of actions listed in Table \ref{tab:static_attackers}, where $I_t$ is the start time of the intrusion. The actions consist of reconnaissance commands and exploits. During each time-step, one action is executed. The three attackers differ in the reconnaissance command that they use and the number of stops $L$ required to prevent the attack (see Table \ref{tab:static_attackers_differences}).

\textsc{NoviceAttacker} uses brute-force attacks to exploit password vulnerabilities (e.g. SSH dictionary attacks) and uses a TCP/UDP port scan for reconnaissance. The attack is prevented if the defender takes a stop action and revokes the user certificates.

\textsc{ExperiencedAttacker} uses a ping scan for reconnaissance and performs both brute-force attacks and more sophisticated attacks, such as a command injection attack (e.g. CVE-2014-6271). The attack is prevented if the defender takes two stop actions and blacklists IP addresses that have been flagged by the IDS in addition to revoking the user certificates.

Lastly, \textsc{ExpertAttacker} only targets vulnerabilities that can be exploited \textit{without} brute-force methods and thus generates less network traffic, for example remote execution vulnerabilities, such as, CVE-2017-7494. The attacker uses a ping scan for reconnaissance like \textsc{ExperiencedAttacker}. The attack is prevented if the defender executes three stop actions and blocks the gateway.

Since the ping-scan generates fewer IDS alerts than the TCP/UDP port scan, the reconnaissance actions of \textsc{ExperiencedAttacker} and \textsc{ExpertAttacker} are harder to detect than those of \textsc{NoviceAttacker}.
\begin{figure}
  \centering
    \scalebox{0.85}{
      \includegraphics{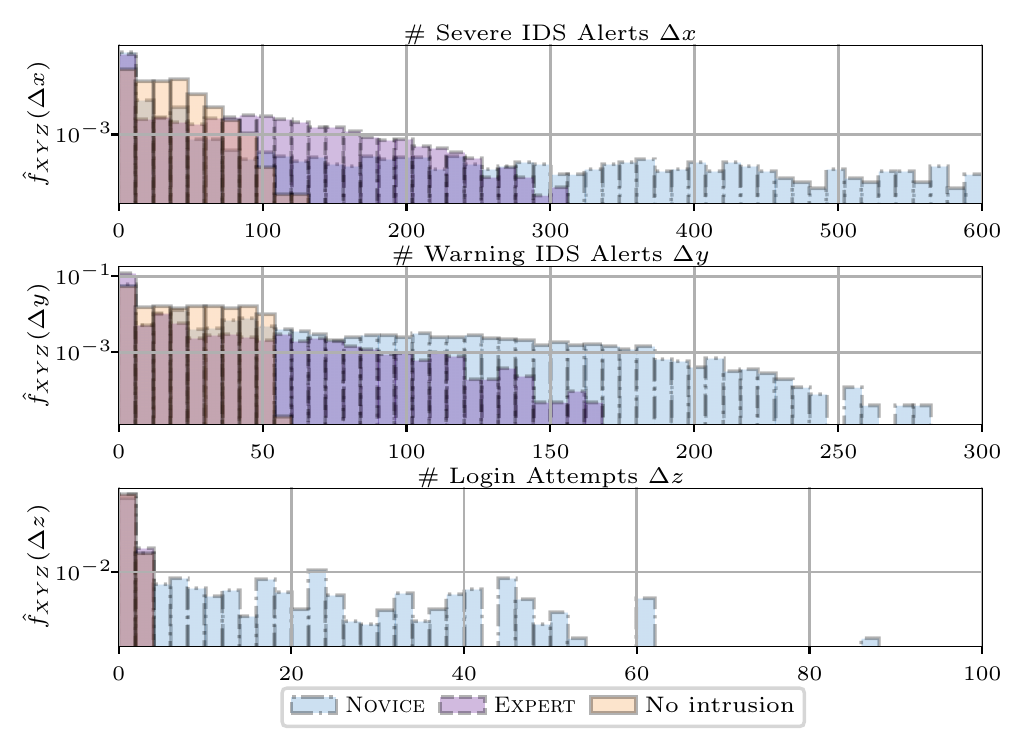}
    }
    \caption{Empirical distributions of severe IDS alerts $\Delta x$ (top row), warning IDS alerts $\Delta y$ (middle row), and login attempts $\Delta z$ (bottom row) generated during time-steps of intrusions by different attackers as well as during time-steps when no intrusion occurs.}
    \label{fig:ids_infra_one_macine_22}
  \end{figure}
\subsection{Estimating the Distributions of Alerts and Login Attempts}
In this section, we describe how we collect data from the emulation system and estimate the distributions of alerts and login attempts.

1) At the end of every time-step, the emulation system collects the metrics $\Delta x$, $\Delta y$, $\Delta z$, which contain the alerts and login attempts that occurred during the time-step. For the evaluation reported in this paper we collected measurements from $21000$ time-steps of $30$ seconds each.

2) From the collected measurements, we compute the empirical distribution $\hat{f}_{XYZ}$ as estimate of the corresponding distribution $f_{XYZ}$ in the target infrastructure. For each state $s_t$, we obtain the conditional distribution $\hat{f}_{XYZ | s_t}$.

Fig. \ref{fig:ids_infra_one_macine_22} shows some of the empirical distributions. The distributions related to \textsc{ExperiencedAttacker} are omitted for better readability. The estimated distributions from \textsc{ExpertAttacker} and \textsc{ExperiencedAttacker} mostly overlap with the distributions obtained when no intrusion occurs. However, a clear difference between the distributions obtained during an intrusion of \textsc{NoviceAttacker} and the distributions when no intrusion occurs can be observed. From these empirical distributions, we note that the assumption that the observation distribution is TP2 in Theorem \ref{thm:structural_result_2}.C is reasonable.
\subsection{Simulating an Episode of the POMDP}\label{sec:simulation_episode}
During a simulation of the POMDP, the system state evolves according to the dynamics described in Section \ref{sec:formal_model_2}, and the observations evolve according to the estimated distribution $\hat{f}_{XYZ}$. In the initial state, no intrusion occurs. During an episode, an intrusion normally occurs at a random start time. It is also possible that the defender performs $L$ stops before the intrusion would start, in which case no intrusion starts.

A simulated episode evolves as follows. The episode starts in state $s_1=0$ and $l_1=L$. During each time-step, the simulation system samples an action from the defender policy $a_t \sim \pi_{\theta,l}(\cdot | b_{t})$. If the action is stop ($a_t=1$) and $l_t=1$, the episode ends. Otherwise, the number of remaining stop actions is updated: $l_{t+1}=l_t-a_t$. Further, if an intrusion is in progress, the system executes an attacker action following Table \ref{tab:static_attackers}. It then updates the state $s_t\rightarrow s_{t+1}$ and samples $\Delta x_{t+1}, \Delta y_{t+1}, \Delta z_{t+1}$ from the empirical distribution $\hat{f}_{XYZ|s_{t+1}}$. (The activities of the clients are not simulated but are captured by $\hat{f}_{XYZ}$.) The simulation then computes the belief $b_{t+1}$ using Eq. \ref{eq:belief_upd} and computes the defender reward $r_{t+1}$ using Eqs. \ref{eq:reward_0}-\ref{eq:reward_3}. (Note that the exact reward can be computed during training and evaluation of policies but not when the policies are deployed in the target infrastructure as it depends on the hidden state.) The sequence of time-steps continues until the defender performs the final stop, after which the episode ends. If the attacker sequence in Table \ref{tab:static_attackers} is completed before the defender performs the final stop, the sequence is restarted.
\subsection{Emulating an Episode of the POMDP}\label{sec:emulation_episode}
Just like a simulated episode, an emulated episode starts with the same initial conditions, evolves in discrete time-steps, and experiences an intrusion event at a random time. However, an episode in the emulation system differs from an episode in the simulation system in the following ways. First, attacker and defender actions in the emulation system include computing and networking functions with real side-effects in the emulation environment (see Table \ref{tab:defender_stop_actions} and Table \ref{tab:static_attackers}). Further, the defender observations in the emulation system are not sampled but are obtained through reading log files and metrics of the emulated infrastructure. Lastly, the emulated client population performs requests to the emulated application servers just like on a real infrastructure (see Section \ref{sec:emu_target_inf}). Due to these differences, running an episode in the emulation system takes much longer time than running a similar episode in the simulation system.
\section{Learning Intrusion Prevention Policies for the Target Infrastructure}\label{sec:eval}
Our approach for finding effective defender policies includes (1) extensive simulation of POMDP episodes in the simulation system to learn the policies; and (2), evaluation of the learned policies through running POMDP episodes in the emulation system. This section describes our evaluation results.

The environment for training policies and running simulations is a Tesla P100 GPU. The hyperparameters for the training algorithm are listed in Appendix \ref{appendix:hyperparameters}. The emulated infrastructure is deployed on a server with a 24-core Intel Xeon Gold 2.10GHz CPU and 768 GB RAM. We have made available the code of our simulation system, as well as the measurement traces used to estimate the observation distributions of the POMDP, which can be used by others to extend and validate our results \cite{github_cnsm_21_hammar_stadler}.
\subsection{Evaluation Process}
\begin{figure*}
  \centering
    \scalebox{0.98}{
      \includegraphics{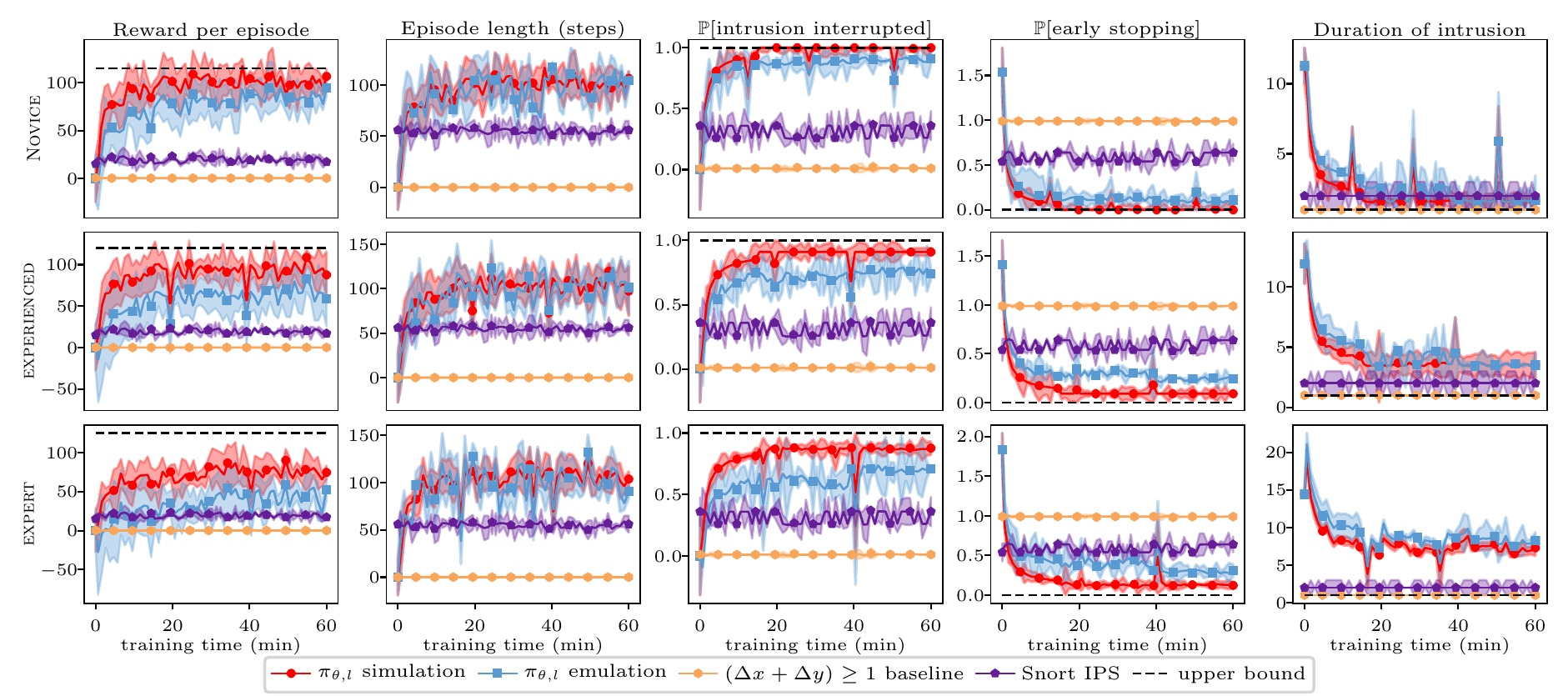}
    }
    \caption{Learning curves obtained during training of \textsc{T-SPSA}; red curves show simulation results and blue curves show emulation results; the purple, orange, and black curves relate to baseline policies; the rows from top to bottom relate to: \textsc{NoviceAttacker}, \textsc{ExperiencedAttacker}, and \textsc{ExpertAttacker}; the columns from left to right show performance metrics: episodic reward, episode length, empirical prevention probability, empirical early stopping probability, and the time between the start of intrusion and the $L$th stop action; the curves show the mean and $95\%$ confidence interval for five training runs with different random seeds.}
    \label{fig:defender_simulation_emulation_multiple_attackers_multiple_stop_tnsm_21_three_stops}
  \end{figure*}

\begin{figure*}
\centering
\scalebox{0.98}{
      \includegraphics{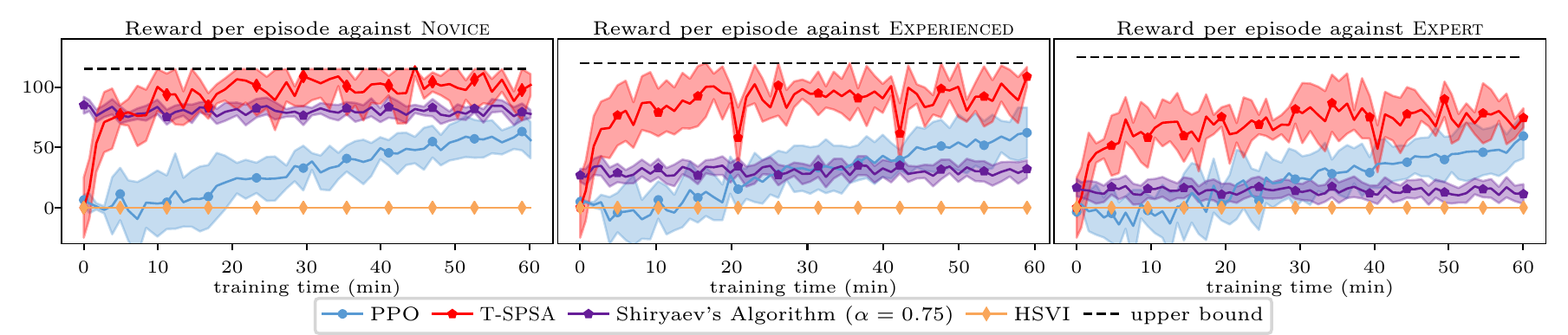}
}
\caption{Comparison between \textsc{T-SPSA} and three baseline algorithms; all curves show simulation results; red curves relate to \textsc{T-SPSA}; blue curves relate to PPO; orange curves relate to HSVI; purple curves relate to Shiryaev's algorithm with threshold $\alpha=0.75$; the columns from left to right relate to: \textsc{NoviceAttacker}, \textsc{ExperiencedAttacker}, and \textsc{ExpertAttacker}; all curves show the mean and $95\%$ confidence interval for five training runs with different random seeds.}
    \label{fig:ppo_spsa_comparison}
  \end{figure*}
We train three defender policies against \textsc{NoviceAttacker}, \textsc{ExperiencedAttacker} and \textsc{ExpertAttacker} until convergence. For each attacker we run $10'000$ training episodes to estimate an optimal defender policy using the method described in Section \ref{sec:rl_approach}. After each episode we evaluate the current defender policy.

To evaluate a defender policy, we run evaluation episodes and compute various performance metrics. Specifically, we run $500$ evaluation episodes in the simulation system and $5$ evaluation episodes in the emulation system.

The $10'000$ training episodes and the evaluation described above constitute one \textit{training run}. We run five training runs with different random seeds. A single training run takes about $4$ hours of processing time on a P100 GPU to perform the simulations and the policy-training, as well as around $12$ hours for evaluating the policies in the emulation system.

We compare the policies learned through \textsc{T-SPSA} with three baseline policies. The first baseline prescribes the stop action whenever an IDS alert occurs, i.e., whenever $(\Delta x+ \Delta y)\geq 1$. The second baseline is obtained by configuring the Snort IDS as an Intrusion Prevention System (IPS) which drops network traffic following its internal recommendation system (see Appendix \ref{appendix:infrastructure_configuration} for the Snort configuration). To calculate the reward, we define $100$ dropped IP packets of the Snort IPS to be a stop action of the defender. Lastly, the third baseline is an ideal policy which presumes knowledge of the exact intrusion time and performs all stop actions at exactly that time.

We evaluate our algorithm, \textsc{T-SPSA}, by comparing it with three baseline algorithms: Proximal Policy Optimization (PPO) \cite{ppo}, Heuristic Search Value Iteration (HSVI) \cite{hsvi}, and Shiryaev's algorithm \cite{shirayev_change_point}. PPO is a state-of-the-art reinforcement learning algorithm, HSVI is a state-of-the-art dynamic programming algorithm for POMDPs, and Shiryaev's algorithm is an optimal algorithm for change detection. The main difference between \textsc{T-SPSA} and the first two baselines (PPO and HSVI) is that \textsc{T-SPSA} exploits the threshold structure expressed in Theorem \ref{thm:structural_result_2} and the main difference in comparison with Shiryaev's algorithm is that \textsc{T-SPSA} learns $L$ thresholds whereas Shiryaev's algorithm uses a single predefined threshold. We set this threshold to $0.75$ based on a hyperparameter search (see Appendix \ref{appendix:hyperparameters}).

\subsection{Learning Intrusion Prevention Policies}\label{sec:one_stop_evaluation}
Fig. \ref{fig:defender_simulation_emulation_multiple_attackers_multiple_stop_tnsm_21_three_stops} shows the performance of the learned policies against the three attacker types. The red curves represent the results from the simulation system and the blue curves show the results from the emulation system. The purple and orange curves give the performance of the Snort IPS baseline and the baseline policy that mandates a stop action whenever an IDS alert occurs, respectively. The dashed black curves give the performance of the baseline policy that assumes knowledge of the exact intrusion time.

An analysis of the graphs in Fig. \ref{fig:defender_simulation_emulation_multiple_attackers_multiple_stop_tnsm_21_three_stops} leads us to the following conclusions. We observe that the learning curves converge quickly to constant mean values for all attackers and across all investigated performance metrics. From this we conclude that the learned policies have converged as well.

Second, we observe that the converged values of the learning curves are close to the dashed black curves, which give an upper bound to any optimal policy. In addition, we see that the empirical probability of preventing an intrusion of the learned policies is close to $1$ (middle column of Fig. \ref{fig:defender_simulation_emulation_multiple_attackers_multiple_stop_tnsm_21_three_stops}) and that the empirical probability of stopping before the intrusion starts is close to $0$ (second rightmost column of Fig. \ref{fig:defender_simulation_emulation_multiple_attackers_multiple_stop_tnsm_21_three_stops}). This suggests that the learned policies are close to optimal. We also observe that all learned policies do significantly better than the Snort IPS baseline and the baseline that stops whenever an IDS alert occurs (leftmost column in Fig. \ref{fig:defender_simulation_emulation_multiple_attackers_multiple_stop_tnsm_21_three_stops}).

Third, although the learned policies, as expected, perform better in the simulation system than in the emulation system, we are encouraged by the fact that the curves of the emulation system are close to those of the simulation system.

We also note from Fig. \ref{fig:defender_simulation_emulation_multiple_attackers_multiple_stop_tnsm_21_three_stops} that the learned policies do best against \textsc{NoviceAttacker} and less well against \textsc{ExperiencedAttacker} and \textsc{ExpertAttacker}. For instance, the learned policies against \textsc{ExperiencedAttacker} and \textsc{ExpertAttacker} are more likely to stop before an intrusion has started (second rightmost column of Fig. \ref{fig:defender_simulation_emulation_multiple_attackers_multiple_stop_tnsm_21_three_stops}). This indicates that \textsc{NoviceAttacker} is easier to detect for the defender as its actions create more IDS alerts than those of the other attackers, as pointed out in Section \ref{sec:emu_target_inf}.

Lastly, Fig. \ref{fig:ppo_spsa_comparison} shows a comparison between our reinforcement learning algorithm (\textsc{T-SPSA}) and the three baseline algorithms in the simulation system. We observe in Fig. \ref{fig:ppo_spsa_comparison} that both \textsc{T-SPSA} and PPO converge to close approximations of an optimal policy within an hour of training whereas HSVI does not converge within the measured time. The slow convergence of HSVI manifests the intractability of using dynamic programming to compute policies in large POMDPs \cite{pspace_complexity}. We also see in Fig. \ref{fig:ppo_spsa_comparison} that \textsc{T-SPSA} converges significantly faster than PPO. This is expected since \textsc{T-SPSA} considers a smaller space of policies than PPO. Finally, we also note in Fig. \ref{fig:ppo_spsa_comparison} that \textsc{T-SPSA} outperforms Shiryaev's algorithm, which demonstrates the benefit of using $L$ thresholds instead of a single threshold.

\section{Related Work}\label{sec:related_work}
Traditional approaches to intrusion prevention use packet inspection and static rules for detection of intrusions and selection of response actions \cite{snort,ids_survey, int_prevention}. Their main drawback lies in the need for domain experts to configure the rule sets. As a consequence, much effort has been devoted to developing methods for finding security policies in an automatic way. This research uses concepts and methods from various areas, most notably from anomaly detection (see example \cite{anomaly_ddetection}), change-point detection (see example \cite{tartakovsky_1}), statistical learning (see examples \cite{fung_ids_distributed,fung_ids_distributed_dirichlet,ml_anomaly_detection}), control theory (see survey \cite{Miehling_control_theoretic_approaches_summary}), game theory (see textbooks \cite{nework_security_alpcan,tambe,carol_book_intrusion_detection,levente_book}), artificial intelligence (see survey \cite{ai_survey}), dynamic programming (see example \cite{dp_security_1}), reinforcement learning (see surveys \cite{deep_rl_cyber_sec,control_rl_reviews}), evolutionary methods (see example \cite{armsrace_malware}), and attack graphs (see example \cite{miehling_attack_graph}).

While the research reported in this paper is informed by all the above works, we limit the following discussion to prior work that centers around finding security policies through reinforcement learning, a topic area that has grown considerably in recent years. Three seminal papers: \cite{rl_seminal}, \cite{rl_seminal_3}, and \cite{rl_seminal_2}, published in $2000$, $2005$, and $2008$, respectively, analyze intrusion prevention use cases and evaluate traditional reinforcement learning algorithms for this task. These papers have inspired much follow-up research, e.g. on studying \textit{deep} reinforcement learning algorithms for intrusion prevention \cite{hammar_stadler_cnsm_20,hammar_stadler_cnsm_21,pentest_rl_rohit} and studying new use cases,  such as defense against jamming attacks \cite{jamming_defense_mowla}, mitigation of denial of service attacks \cite{Malialis2013MultiagentRT,ddos_mitigation_rl}, defense against advanced persistent threats \cite{rl_apts}, placement of honeypots \cite{honeypot_placement_zhu_1}, botnet detection \cite{botnet_detection_1, rl_botnet_evasion}, detection of flip attacks \cite{rl_qcd_kalle}, detection of network traffic anomalies \cite{rl_abnormal_traffic_detection}, greybox fuzzing \cite{fuzzing_rl_nddss_2021_wang}, and defense against topology attacks \cite{topology_attacks_defense}.

Among the recent works that use reinforcement learning to find security policies, many focus on intrusion prevention use cases similar to the one we discuss in this paper \cite{hammar_stadler_cnsm_20,hammar_stadler_cnsm_21, elderman, schwartz_2020, oslo_pentest_rl, kurt_rl, microsoft_red_teaming, ridley_ml_defense, rl_cyberdefense_heartbleed, deep_hierarchical_rl_pentest, pentest_rl_rohit, adaptive_cyber_defense_pomdp_rl, muzero_sdn,atmos,sdn_rl_ddos,deep_air}. These works use a variety of models, including MDPs \cite{oslo_pentest_rl,ridley_ml_defense,deep_hierarchical_rl_pentest,pentest_rl_rohit,deep_air}, Markov games \cite{elderman, hammar_stadler_cnsm_20, muzero_sdn}, and POMDPs \cite{hammar_stadler_cnsm_21,adaptive_cyber_defense_pomdp_rl,kurt_rl}, as well as various reinforcement learning algorithms, including Q-learning \cite{elderman,oslo_pentest_rl,ridley_ml_defense}, SARSA \cite{kurt_rl}, PPO \cite{hammar_stadler_cnsm_20,hammar_stadler_cnsm_21}, hierarchical reinforcement learning \cite{deep_hierarchical_rl_pentest}, DQN \cite{pentest_rl_rohit}, Thompson sampling \cite{adaptive_cyber_defense_pomdp_rl}, MuZero \cite{muzero_sdn}, NFQ \cite{atmos}, DDQN \cite{deep_air}, and DDPG \cite{sdn_rl_ddos}.

This paper differs from the works referenced above in two main ways. First, we formulate the intrusion prevention problem as a multiple stopping problem. The other works formulate the problem as solving a general MDP, POMDP, or Markov game. The advantage of our approach is that we obtain structural properties of optimal policies, which have practical benefits (see Section \ref{sec:dp_opt}).

Problem formulations based on optimal stopping theory can be found in prior research on change detection \cite{shirayev_change_point,page_1,kurt_rl,tartakovsky_1,rl_qcd_kalle,hammar_stadler_cnsm_21}. Compared to these papers, our approach is more general by allowing multiple stop actions within an episode. Another difference is that we model intrusion \textit{prevention} rather than intrusion \textit{detection}. Further, compared with traditional change detection algorithms, e.g. CUSUM \cite{page_1} and Shiryaev's algorithm \cite{shirayev_change_point}, our algorithm \textit{learns} thresholds and does not assume them to be preconfigured.

Second, our solution method to find effective policies for intrusion prevention includes using an emulation system in addition to a simulation system. The advantage of our method compared to the  simulation-only approaches \cite{hammar_stadler_cnsm_20,hammar_stadler_cnsm_21, elderman, schwartz_2020, oslo_pentest_rl, kurt_rl, microsoft_red_teaming, ridley_ml_defense, rl_cyberdefense_heartbleed, deep_hierarchical_rl_pentest, pentest_rl_rohit, adaptive_cyber_defense_pomdp_rl} is that the parameters of our simulation system are determined by measurements from an emulation system instead of being chosen by a human expert. Further, the learned policies are evaluated in the emulation system, not in the simulation system. As a consequence, the evaluation results give higher confidence of the obtained policies' performance in the target infrastructure than what simulation results would provide.

Some prior works on reinforcement learning for intrusion prevention that make use of emulation are: \cite{muzero_sdn}, \cite{atmos}, \cite{sdn_rl_ddos}, and \cite{deep_air}. They emulate software-defined networks based on Mininet \cite{mininet}.  The main differences between these efforts and the work described in this paper are: (1) we develop our own emulation system which allows for experiments with a large variety of exploits; (2) we focus on a different intrusion prevention use case; (3) we do not assume that the defender has perfect observability; and (4), we use an underlying theoretical framework to formalize the use case, derive structural properties of optimal policies, and test these properties in an emulation system.

Finally, \cite{cyborg} and \cite{farland} describe ongoing efforts in building emulation platforms for reinforcement learning, which resemble our emulation system. In contrast to these papers, our emulation system has been built to investigate the specific use case of intrusion prevention and forms an integral part of our general solution method (see Fig. \ref{fig:method}).

\section{Conclusion and Future Work}\label{sec:conclusions}
In this paper, we proposed a novel formulation of the intrusion prevention problem based on the theory of optimal stopping. This formulation allowed us to derive that a threshold policy based on infrastructure metrics is optimal, which has several practical benefits.

To find and evaluate policies, we used a reinforcement learning method that includes a simulation system and an emulation system. In contrast to a simulation-only approach, our method produces policies that can be executed in a target infrastructure.

Through extensive evaluations, we showed that our approach can produce effective defender policies for a practical configuration of an IT infrastructure (Figs. \ref{fig:defender_simulation_emulation_multiple_attackers_multiple_stop_tnsm_21_three_stops}-\ref{fig:ppo_spsa_comparison}). We also demonstrated that our reinforcement learning algorithm (\textsc{T-SPSA}), which takes advantage of the threshold structure (Theorem \ref{thm:structural_result_2}), outperforms state-of-the-art algorithms on our use case.

We make assumptions in this paper that limit the practical applicability of the results: the attacker follows a static policy, and the defender learns only the times of taking defensive actions but not the types of actions. Therefore, the question arises whether our approach can be extended so that (1) the attacker can pursue a wide range of realistic policies and (2) the defender learns optimal policies that express not only when defensive actions needs to be taken but also the specific measure to be executed.

Addressing these points is part of our research agenda. The dynamic attacker can be studied using a game-theoretic extension of the introduced framework. The theory tells us that an optimal solution can be found through self-play in a similar manner as described in the paper, but further work is needed to show that such a solution is feasible in practice. Scenarios involving several attackers can also be studied in this context.

We also plan to extend the defender model to include the selection of defensive actions. One possible approach is to learn two orthogonal policies: a policy that decides when to take a defensive action and another policy that decides which action to take.

\section{Acknowledgments}
This research has been supported in part by the Swedish armed forces and was conducted at KTH Center for Cyber Defense and Information Security (CDIS). The authors would like to thank Pontus Johnson for his useful input to this research and Vikram Krishnamurthy for helpful discussions. The authors are also grateful to Forough Shahab Samani and Xiaoxuan Wang for their constructive comments on a draft of this paper.

\appendices
\section{Proof of Theorem \ref{thm:structural_result_2}.}\label{appendix:proof_structural_result_2}
Given the POMDP introduced in Section \ref{sec:pomdp_model}, let $L$ denote the number of stop actions, $f_{XYZ}$ the observation distribution, $\mathcal{B} = [0,1]$ the belief space (see Section \ref{sec:dp_opt}), $b(1)$ the belief state, $\mathscr{S}_{l}$ the stopping set, and $\mathscr{C}_{l}$ the continuation set.

We use the value iteration algorithm to establish structural properties of $V_l^{*}$ and $\pi_l^{*}$ \cite{puterman,krishnamurthy_2016}. Let $V_l^k$, $\mathscr{S}^k_l$, and $\mathscr{C}^k_l$, denote the value function, the stopping set, and the continuation set at iteration $k$ of the value iteration algorithm, respectively. Let $V_l^0\big(b(1)\big)=0$ for $b(1)\in [0,1]$ and $l\in \{1,\hdots,L\}$. Then, $\lim_{k\rightarrow \infty}V_l^k=V_l^{*}, \lim_{k\rightarrow \infty}\mathscr{S}^k_l=\mathscr{S}_l$, and $\lim_{k\rightarrow \infty}\mathscr{C}^k_l=\mathscr{C}_l$ \cite{puterman,krishnamurthy_2016}.

The main idea behind the proof of Theorem \ref{thm:structural_result_2} is to show that the stopping sets $\mathscr{S}_{l}$ have the form $\mathscr{S}_{l} = [\alpha_l^{*}, 1] \subseteq \mathcal{B}$ and that $\alpha_l^{*} \geq \alpha_{l+1}^{*}$ for $l\in \{1,\hdots,L\}$. Towards this goal, we state the following four lemmas.
\begin{lemma}\label{lemma:stops_required}
During a POMDP episode, an optimal policy $\pi_{L}^{*}$ prescribes $L$ stop actions.
\end{lemma}
\begin{proof}[Proof.]
The proof follows directly from the definition of the transition probabilities (see Eqs. \ref{eq:tp_1}-\ref{eq:tp_4}) and the reward function (see Eqs. \ref{eq:reward_0}-\ref{eq:reward_3}).
\end{proof}
\begin{lemma}\label{lemma:convex_stopping_set}
$\mathscr{S}_{1}$ is a convex subset of $\mathcal{B}$.
\end{lemma}
\begin{proof}[Proof.]
The proof can be found in \cite[pp. 10, Lemma 3]{hammar_stadler_cnsm_21} and in \cite[pp. 258, Theorem 12.2.1]{krishnamurthy_2016}.
\end{proof}
\begin{lemma}\label{lemma:p_r_tp2}
$\mathcal{P}^{a_t}_{s_t,s_{t+1},l_t}$ is TP2 and $\mathcal{R}^{S}_{b(1),l_t}-\mathcal{R}^{C}_{b(1),l_t}$ is increasing in $b(1)$ for $l_t\in \{1,\hdots,L\}$.
\end{lemma}
\begin{proof}[Proof.]
The transition probabilities (see Section \ref{sec:pomdp_model}) are given by the following two row-stochastic matrices:
\begin{align}
\kbordermatrix{
    & 0 & 1 & \emptyset \\
  0 & 0.99 & 0.01 & 0\\
  1 & 0 & 1 & 0\\
  \emptyset & 0 & 0 & 1
},\quad\quad
\kbordermatrix{
    & 0 & 1 & \emptyset \\
  0 & 0 & 0 & 1\\
  1 & 0 & 0 & 1\\
  \emptyset & 0 & 0 & 1
},
\end{align}
The left matrix corresponds to the transition probabilities when $a_t=C$, or, when $a_t=S$ and $l_t>1$. The right matrix represents the transition probabilities when $a_t=S$ and $l_t=1$. To show that $\mathcal{P}^{a_t}_{s_t,s_{t+1},l_t}$ is TP2, it is sufficient to show that all $\binom{3}{2}^2$ second order minors of both matrices are non-negative. The second-order minors of the first matrix are $M_{1,2}=M_{1,3}=M_{2,3}=M_{3,1}=M_{3,2}=0$, $M_{1,1}=1$, $M_{2,1}=0.01$, $M_{2,2}=M_{3,3}=0.99$, where $M_{i,j}$ denotes the determinant of the submatrix formed by deleting the $i$th row and $j$th column. For the second matrix all second order minors are zero. Hence, $\mathcal{P}^{a_t}_{s_t,s_{t+1},l_t}$ is TP2.

$\mathcal{R}^{S}_{b(1),l_t}-\mathcal{R}^{C}_{b(1),l_t}$ is expanded to:
\begin{align}
\mathcal{R}^{S}_{b(1),l_t}-\mathcal{R}^{C}_{b(1),l_t} &= b(1)\left(\frac{50}{4l_t} + 10/L\right) - 1
\end{align}
which is increasing in $b(1)$.
\end{proof}
\begin{lemma}\label{lemma:monotone_filter}
  Given two beliefs $b^{\prime}(1) \geq b(1)$ and two observations $o \geq \bar{o}$, if $\mathcal{P}^{a_t}_{s_t,s_{t+1},l_t}$ and $f_{XYZ|s}$ are TP2, then the following holds for any $a \in \mathcal{A}$, $k \in \mathcal{O}$, and $l_t \in \{1,\hdots, L\}$:
  \begin{enumerate}
  \item $b_{a}^{\prime,o}(1) \geq b_{a}^{o}(1)$
  \item $\mathbb{P}[o \geq k |b^{\prime},a] \geq \mathbb{P}[o \geq k |b,a]$
  \item $b_{a}^{o}(1) \geq b_{a}^{\bar{o}}(1)$
  \end{enumerate}
where $b^{\prime,o}_{a}(1)$ and $b^o_{a}(1)$ denote the beliefs updated with Eq. \ref{eq:belief_upd} after taking action $a \in \mathcal{A}$ and observing $o \in \mathcal{O}$.
\end{lemma}
\begin{proof}[Proof.]
The proof is published in \cite[Theorem 10.3.1, pp. 225,238]{krishnamurthy_2016}. (Remark: in the referenced proof, the monotone likelihood ratio (MLR) order is considered; in our case $|\mathcal{S}\setminus \emptyset|=2$, hence the MLR order reduces to the natural order $b^{\prime}(1) \geq b(1)$.)
\end{proof}
\begin{proof}[Proof of Theorem \ref{thm:structural_result_2}.A.]
The proof has originally been published in \cite[Propositions 4.5-4.8, pp. 437-441]{Nakai1985}. It is also available in a more accessible form in \cite[Theorem 1.C, Theorem 8, pp. 389-397]{optimal_multiple_stopping_social_media_1}. We give our own version of the proof since the referenced proofs assume zero reward for the continue action and assume that rewards are independent of $l$.

If $b(1) \in \mathscr{S}_{l-1}$, the Bellman equation and the fact that $\mathbb{P}[o|a,b]=\mathbb{P}[o|b]=\mathbb{P}^o_{b(1)}$ for all $a \in \mathcal{A}$ and $o\neq \emptyset$ (see Eq. \ref{eq:obs_function}) implies that:
\begin{align}
  &\mathcal{R}^S_{b(1),l-1} - \mathcal{R}^C_{b(1),l-1} + \label{eq:proof_ind_1}\\
&  \sum_{o \in \mathcal{O}} \mathbb{P}^o_{b(1)}\Big(V^{*}_{l-2}\big(b^o(1)\big) - V^{*}_{l-1}\big(b^o(1)\big)\Big)\geq 0\nonumber
\end{align}
\normalsize We show that $b(1) \in \mathscr{S}_{l}$ follows from the above inequality.

Let $W^k_l\big(b(1)\big) = \mathcal{R}^S_{b(1),l} - \mathcal{R}^C_{b(1),l} + V^{k}_{l-1}\big(b(1)\big) - V^{k}_{l}\big(b(1)\big)$. To show that $b(1) \in \mathscr{S}_{l-1} \implies b(1) \in \mathscr{S}_{l}$, it is sufficient to show that $W^k_l(b(1))$ is non-decreasing in $l$ for all $k\geq 0$. We proceed to show this statement by mathematical induction.

For iteration $k=0$ of value iteration, $W^0_l\big(b(1)\big)=V^{0}_{l}\big(b(1)\big)-V^{0}_{l-1}\big(b(1)\big)=0$, which is trivially non-decreasing in $l$. Assume by induction that $W^{k-1}_{l}\big(b(1)\big)$ is non-decreasing in $l$ for iterations $k-1,k-2,\hdots,1$. To show that $W^{k}_l\big(b(1)\big)$ is non-decreasing in $l$ also for iteration $k$, we show that $W^{k}_{l}\big(b(1)\big) - W^{k}_{l-1}\big(b(1)\big) \geq 0$.

There are four cases to consider:
\begin{enumerate}
\item If $b(1) \in \mathscr{S}^k_l \cap \mathscr{S}^k_{l-1} \cap \mathscr{S}^k_{l-2}$, then:
\begin{align}
         &W^{k}_{l}\big(b(1)\big) - W^{k}_{l-1}\big(b(1)\big) \\
         &= \sum_{o \in \mathcal{O}}\mathbb{P}_{b(1)}^{o}\Big(W^{k-1}_{l-1}\big(b^o(1)\big)-W^{k-1}_{l-2}\big(b^o(1)\big)\Big)\nonumber
\end{align}\normalsize
which is non-negative by the induction assumption.
\item If $b(1) \in \mathscr{S}^k_l \cap \mathscr{S}^k_{l-1} \cap \mathscr{C}^k_{l-2}$, then:
\begin{align}
&W^{k}_{l}\big(b(1)\big) - W^{k}_{l-1}\big(b(1)\big) =\mathcal{R}_{b(1),l-1}^{S} - \mathcal{R}_{b(1),l-1}^{C}\\
         & + \sum_{o \in \mathcal{O}}\mathbb{P}_{b(1)}^{o}\Big(V^{k-1}_{l-2}\big(b^o(1)\big) - V^{k-1}_{l-1}\big(b^o(1)\big)\Big)\nonumber
\end{align}\normalsize
which is non-negative because $b(1) \in \mathscr{S}^k_{l-1}$ (it is implied by Eq. \ref{eq:optimal_stopping_1}).

\item If $b(1) \in \mathscr{S}^k_l \cap \mathscr{C}^k_{l-1} \cap \mathscr{C}^k_{l-2}$, then:
\begin{align}
&W^{k}_{l}\big(b(1)\big) - W^{k}_{l-1}\big(b(1)\big) = \mathcal{R}_{b(1),l-1}^{C} - \mathcal{R}_{b(1),l-1}^{S}\label{eq:kalai_23}\\
         & +\sum_{o \in \mathcal{O}}\mathbb{P}_{b(1)}^{o}\Big(V^{k-1}_{l-1}\big(b^o(1)\big) - V^{k-1}_{l-2}\big(b^o(1)\big)\Big)\nonumber
\end{align}\normalsize
which is non-negative because $b(1) \in \mathscr{C}^k_{l-1}$ (it is implied by Eq. \ref{eq:optimal_stopping_1}).

\item If $b(1) \in \mathscr{C}^k_l \cap \mathscr{C}^k_{l-1} \cap \mathscr{C}^k_{l-2}$, then:
 \begin{align}
&W^{k}_{l}\big(b(1)\big) - W^{k}_{l-1}\big(b(1)\big) =\\
& \sum_{o \in \mathcal{O}}\mathbb{P}_{b(1)}^{o}\Big(W^{k-1}_{l}\big(b^o(1)\big)-W^{k-1}_{l-1}\big(b^o(1)\big)\Big)
\end{align}\normalsize
which is non-negative by the induction assumption.
\end{enumerate}
The other cases, e.g. $b(1) \in \mathscr{C}^k_{l} \cap \mathscr{C}^k_{l-1}\cap \mathscr{S}^k_{l-2}$, can be discarded due to the induction assumption. Hence, $W^{k}_l(b(1))$ is non-decreasing in $l$ for all $k \geq 0$.

Since the left-hand side of Eq. \ref{eq:proof_ind_1} is non-decreasing in $l$ it follows that if Eq. \ref{eq:proof_ind_1} holds, i.e. if $b(1) \in \mathscr{S}_{l-1}$, then $b(1) \in \mathscr{S}_{l}$.
\end{proof}
\begin{proof}[Proof of Theorem \ref{thm:structural_result_2}.B.]
The proof follows the chain of reasoning in \cite[Corollary 12.2.2, pp. 258]{krishnamurthy_2016}.

Using Lemma \ref{lemma:convex_stopping_set}, we know that the stopping set $\mathscr{S}_{1}$ is a convex subset of $\mathcal{B}=[0,1]$. That is, it has the form $[\alpha^{*}, \beta^{*}]$ where $0 \leq \alpha^{*} \leq \beta^{*} \leq 1$. We show that $\beta^{*} = 1$.

If $b(1) = 1$, the Bellman equation (Eq. \ref{eq:optimal_stopping_1}) states that:
\begin{align}
&\pi_{1}^{*}(1) \in \argmax_{\{S,C\}} \Bigg[\underbrace{50 + V^{*}_{0}(\emptyset)}_{a=S}, \nonumber\\
&\underbrace{-9 + \sum_{o\in \mathcal{O}}\mathcal{Z}(o,1,C)V_{1}^{*}\big(b_C^o(1)\big)}_{a=C}\Bigg]\label{eq:bellman_proof_thm_1_b}
\end{align}\normalsize
As $L=1$, it follows from Lemma \ref{lemma:stops_required} that an optimal policy prescribes one stop action during a POMDP episode and that the intrusion is prevented after the first stop. Hence, $V^{*}_{0}(\emptyset)=\mathcal{R}^{\cdot}_{\emptyset,0}=0$. Moreover, since $s=1$ is an absorbing state until the stop, it follows from the definition of $b_C^o$ (Eq. \ref{eq:belief_upd}) that $b_C^o(1)=1$ for all $o \in \mathcal{O} \setminus \emptyset$. Thus, since $V_{1}^{*}(1) \leq 50$ (see Eqs. \ref{eq:reward_0}-\ref{eq:reward_3}), we get:
\begin{align}
&\pi_{1}^{*}(1) \in \argmax_{\{S,C\}} \Bigg[\underbrace{50}_{a=S}, \underbrace{-9 + V_{1}^{*}(1)}_{a=C}\Bigg] = S
\end{align}\normalsize
This means that $\beta^{*}=1$ and therefore $\mathscr{S}_{1} = [\alpha^{*}, 1]$.
\end{proof}

\begin{corollary}\label{corollary:connected_stopping_set}
If $\mathcal{P}^{a_t}_{s_t,s_{t+1},l_t}$ and $f_{XYZ|s}$ are TP2, the stopping set $\mathscr{S}_{l}$ is connected, $l \in \{1,\hdots,L\}$.
\end{corollary}
\begin{proof}
We adapt the proof from \cite[Theorem 1.B, pp. 389-397]{optimal_multiple_stopping_social_media_1} to our model. In contrast to the referenced proof, our model includes non-zero rewards for the continue action and $|\mathcal{S}\setminus \emptyset |=2$.

If $b(1) \in \mathscr{S}_{l}$, the Bellman equation and the fact that $\mathbb{P}[o|a,b]=\mathbb{P}[o|b]=\mathbb{P}^o_{b(1)}$ for all $a \in \mathcal{A}$ and $o\neq \emptyset$ (see Eq. \ref{eq:obs_function}) implies that:
\begin{align}
\mathcal{R}^S_{b(1),l} - \mathcal{R}^C_{b(1),l} + \sum_{o \in \mathcal{O}} \mathbb{P}^o_{b(1)}\Big(V^{*}_{l-1}\big(b^o(1)\big) - V^{*}_{l}\big(b^o(1)\big)\Big) \geq 0\label{eq:proof_ind_1111}
\end{align}
\normalsize We show that the above inequality implies that $b^{\prime}(1) \in \mathscr{S}_{l}$ for any $b^{\prime}(1) \geq b(1)$, which means that $\mathscr{S}_l$ is connected.

Since $\mathcal{B}=[0,1]$, the beliefs are totally ordered according to the standard ordering. Further, since $f_{XYZ|s}$ is TP2 by assumption and $\mathcal{P}^{a_t}_{s_t,s_{t+1},l_t}$ is TP2 by Lemma \ref{lemma:p_r_tp2}, $b^o(1)$ is weakly increasing in both $b(1)$ and $o \in \mathcal{O}$. Further, $\mathbb{P}[o \geq k |b^{\prime},a] \geq \mathbb{P}[o \geq k |b,a]$ for any $k \in \mathcal{O}$ (Lemma \ref{lemma:monotone_filter}). Thus, since $\mathscr{S}_{l-1}\subseteq \mathscr{S}_{l}$ (Theorem \ref{thm:structural_result_2}.A) and $\mathscr{S}_{1}=[\alpha^{*}_{1},1]$ (Theorem \ref{thm:structural_result_2}.B), it is sufficient to show that $\mathcal{R}^S_{b(1),l} - \mathcal{R}^C_{b(1),l} + V^{*}_{l-1}\big(b(1)\big) - V^{*}_{l}\big(b(1)\big)$ is weakly increasing in $b(1)$. We proceed to show this by mathematical induction.

For iteration $k=0$ of value iteration, $\mathcal{R}^S_{b(1),l} - \mathcal{R}^C_{b(1),l} + V^{0}_{l-1}\big(b^o(1)\big) - V^{0}_{l}\big(b^o(1)\big) = \mathcal{R}^S_{b(1),l} - \mathcal{R}^C_{b(1),l}$ which is weakly increasing in $b(1)$ by Lemma \ref{lemma:p_r_tp2}. Assume by induction that the expression is weakly increasing in $b(1)$ for iterations $k-1,k-2,\hdots,1$. We show that this implies that the induction assumption holds also for iteration $k$.

Since $\mathscr{S}_{l-1}\subseteq \mathscr{S}_{l}$ (Theorem \ref{thm:structural_result_2}.A) and $\mathscr{S}_{1}=[\alpha^{*}_{1},1]$ (Theorem \ref{thm:structural_result_2}.B), there are three cases to consider
\begin{enumerate}
\item If $b(1) \in \mathscr{S}_{l} \cap \mathscr{S}_{l-1}$, then:
  \begin{align}
    &\mathcal{R}^S_{b(1),l} - \mathcal{R}^C_{b(1),l} + V^{k}_{l-1}\big(b(1)\big) - V^{k}_{l}\big(b(1)\big) = \mathcal{R}^S_{b(1),l-1} -\nonumber\\
    & \mathcal{R}^C_{b(1),l-1} + \sum_{o \in \mathcal{O}} \mathbb{P}^{o}_{b(1)}\Big(V^{k-1}_{l-2}\big(b^o(1)\big) - V^{k-1}_{l-1}\big(b^o(1)\big)\Big)
  \end{align}
which is weakly increasing in $b(1)$ by the induction assumption.
\item If $b(1) \in \mathscr{S}_{l} \cap \mathscr{C}_{l-1}$, then:
  \begin{align}
    &\mathcal{R}^S_{b(1),l} - \mathcal{R}^C_{b(1),l} + V^{k}_{l-1}\big(b(1)\big) - V^{k}_{l}\big(b(1)\big) = \\
    &\sum_{o \in \mathcal{O}} \mathbb{P}^{o}_{b(1)}\Big(V^{k-1}_{l-1}\big(b^o(1)\big) - V^{k-1}_{l-1}\big(b^o(1)\big)\Big)=0\nonumber
  \end{align}
which is trivially weakly increasing in $b(1)$.
\item If $b(1) \in \mathscr{C}_{l} \cap \mathscr{C}_{l-1}$, then:
  \begin{align}
    &\mathcal{R}^S_{b(1),l} - \mathcal{R}^C_{b(1),l} + V^{k}_{l-1}\big(b(1)\big) - V^{k}_{l}\big(b(1)\big) = \mathcal{R}^S_{b(1),l} -\nonumber\\
    &\mathcal{R}^C_{b(1),l} + \sum_{o \in \mathcal{O}} \mathbb{P}^{o}_{b(1)}\Big(V^{k-1}_{l-1}\big(b^o(1)\big) - V^{k-1}_{l}\big(b^o(1)\big)\Big)
  \end{align}
which is weakly increasing in $b(1)$ by the induction assumption.
\end{enumerate}
\end{proof}

\begin{proof}[Proof of Theorem \ref{thm:structural_result_2}.C.]
Since $\mathcal{B} = [0,1]$ (see Section \ref{sec:dp_opt}), $f_{XYZ|s}$ is TP2 by assumption, $\mathcal{P}^{a_t}_{s_t,s_{t+1},l_t}$ is TP2 by Lemma \ref{lemma:p_r_tp2}, and $\mathcal{R}^{S}_{b(1),l}-\mathcal{R}^{C}_{b(1),l}$ is increasing in $b(1)$ (Lemma \ref{lemma:p_r_tp2}), it follows from Corollary \ref{corollary:connected_stopping_set} that $\mathscr{S}_{l}$ is a connected subset of $[0,1]$ for $l\in \{1,\hdots, L\}$. Further, from Theorem \ref{thm:structural_result_2}.B we know that $\mathscr{S}_{1} = [\alpha_{1}^{*}, 1]$. Then, because $\mathscr{S}_{l} \subseteq \mathscr{S}_{l+1}$ for $l\in \{1, \hdots, L-1\}$ (Theorem \ref{thm:structural_result_2}.A), we conclude that $\mathscr{S}_{l} = [\alpha^{*}_{l}, 1]$ and that $\alpha^{*}_l \geq \alpha^{*}_{l+1}$ for $l\in \{1, \hdots, L-1\}$.
\end{proof}
\section{Hyperparameters}\label{appendix:hyperparameters}
\begin{table}
\centering
\resizebox{1\columnwidth}{!}{%
\begin{tabular}{ll} \toprule
  {\textit{Hyperparameters for the POMDP}} & {\textit{Values}} \\
  \hline
  $\gamma$, $\Delta x_{max}, \Delta y_{max},\Delta z_{max}$ & $1$, $6\cdot 10^{2}$, $3\cdot 10^{2}$, $10^{2}$\\
  \\
  {\textit{Hyperparameters for \textsc{T-SPSA}}} & {\textit{Values}} \\
  \hline
  $c, \gamma, \epsilon, A, a$ & $1$, $0.101$, $0.602$, $100$, $1$\\
  \\
  {\textit{Hyperparameters for PPO}} & {\textit{Values}} \\
  \hline
  lr $\alpha$, batch, \# layers, \# neurons, clip $\epsilon$ & $10^{-4}$, $4\cdot 10^{3}t$, $2$, $32$, $0.2$\\
  GAE $\lambda$, ent-coef, activation & $0.95$, $10^{-4}$, ReLU \\
  \\
  {\textit{Hyperparameters for HSVI}} & {\textit{Values}} \\
  \hline
  $\epsilon$ & $0.01$\\
  \\
  {\textit{Hyperparameters for Shiryaev's algorithm}} & {\textit{Values}} \\
  \hline
  $\alpha$ & $0.75$\\
  \bottomrule\\
\end{tabular}
}
\caption{Hyperparameters of the POMDP and the algorithms used for evaluation.}\label{tab:hyperparams}
\end{table}
The hyperparameters used for the evaluation are listed in Table \ref{tab:hyperparams} and were obtained through grid search.
\section{Configuration of the Infrastructure in Fig. \ref{fig:system2}}\label{appendix:infrastructure_configuration}
The configuration of the target infrastructure (Fig. \ref{fig:system2}) is available in Table \ref{tab:emulation_setup}.
\begin{table}
\centering
\resizebox{1\columnwidth}{!}{%
\begin{tabular}{ll} \toprule
  {\textit{ID (s)}} & {\textit{OS:Services:Exploitable Vulnerabilities}} \\ \midrule
  $N_1$ & Ubuntu20:Snort(community ruleset v2.9.17.1),SSH:- \\
  $N_2$ & Ubuntu20:SSH,HTTP Erl-Pengine,DNS:SSH-pw\\
  $N_4$ & Ubuntu20:HTTP Flask,Telnet,SSH:Telnet-pw \\
  $N_{10}$ &Ubuntu20:FTP,MongoDB,SMTP,Tomcat,TS3,SSH:FTP-pw \\
  $N_{12}$ & Jessie:TS3,Tomcat,SSH:CVE-2010-0426,SSH-pw \\
  $N_{17}$ & Wheezy:Apache2,SNMP,SSH:CVE-2014-6271 \\
  $N_{18}$ & Deb9.2:IRC,Apache2,SSH:SQL Injection \\
  $N_{22}$ & Jessie:PROFTPD,SSH,Apache2,SNMP:CVE-2015-3306 \\
  $N_{23}$ & Jessie:Apache2,SMTP,SSH:CVE-2016-10033 \\
  $N_{24}$ & Jessie:SSH:CVE-2015-5602,SSH-pw \\
  $N_{25}$ & Jessie: Elasticsearch,Apache2,SSH,SNMP:CVE-2015-1427\\
  $N_{27}$ & Jessie:Samba,NTP,SSH:CVE-2017-7494\\
  $N_3$,$N_{11}$,$N_{5}$-$N_9$& Ubuntu20:SSH,SNMP,PostgreSQL,NTP:-\\
  $N_{13-16}$,$N_{19-21}$,$N_{26}$,$N_{28-31}$& Ubuntu20:NTP, IRC, SNMP, SSH, PostgreSQL:-\\
  \bottomrule\\
\end{tabular}
}
\caption{Configuration of the target infrastructure (Fig. \ref{fig:system2}).}\label{tab:emulation_setup}
\end{table}

\section{The \textsc{T-SPSA} Algorithm}\label{appendix:spsa_thresholds}
Algorithm \ref{alg:spsa_thresholds} contains the pseudocode of \textsc{T-SPSA}.
\begin{algorithm}
  \caption{\textsc{T-SPSA}}\label{alg:spsa_thresholds}
  \hspace*{\algorithmicindent} \textbf{Input} \\
  \hspace*{\algorithmicindent}  $\mathcal{M}_{\mathcal{P}}, \theta_{(1)} \in \mathbb{R}^{L}$: the POMDP, initial $L$ thresholds\\
  \hspace*{\algorithmicindent}  $N$: number of iterations\\
  \hspace*{\algorithmicindent}  $a,c,\lambda,A,\epsilon$: scalar coefficients\\
  \hspace*{\algorithmicindent} \textbf{Output} \\
  \hspace*{\algorithmicindent} $\theta_{(N+1)}$: learned threshold vector
\begin{algorithmic}[1]
  \Procedure{T-SPSA}{$\mathcal{M}_{\mathcal{P}}$, $\theta_{(1)}$, $N$, $a$, $c$, $\lambda$, $A$, $\epsilon$}
  \For{$n \in \{1, \hdots, N\}$}
  \State $a_n \leftarrow \frac{a}{(n + A)^{\epsilon}}, c_n \leftarrow \frac{c}{n^{\lambda}}$
  \For{$i \in \{1, \hdots, L\}$}
  \State $(\Delta_n)_i \sim \mathcal{U}(\{-1,1\})$
  \EndFor
  \State $R_{high} \sim \hat{J}(\theta_{(n)} + c_n\Delta_n)$, $R_{low} \sim \hat{J}(\theta_{(n)} - c_n\Delta_n)$
  \For{$i \in \{1, \hdots, L\}$}
  \State $\left(\hat{\nabla}_{\theta_{(n)}}J(\theta_{(n)})\right)_{i} \leftarrow \frac{R_{high} - R_{low}}{2c_n(\Delta_n)_{i}}$
  \EndFor
  \State $\theta_{(n+1)} \leftarrow \theta_{(n)} + a_n\hat{\nabla}_{\theta_{(n)}}J(\theta_{(n)})$
  \EndFor
  \State \Return $\theta_{(N+1)}$
\EndProcedure
\end{algorithmic}
\end{algorithm}

\ifCLASSOPTIONcaptionsoff
  \newpage
\fi

\bibliographystyle{IEEEtran}
\bibliography{references,url}

\end{document}